\documentclass{article}

\usepackage{PRIMEarxiv}

\usepackage{natbib}
 \bibpunct[, ]{(}{)}{,}{a}{}{,}%

\usepackage[utf8]{inputenc} 
\usepackage[T1]{fontenc}    
\usepackage{hyperref}       
\usepackage{url}            
\usepackage{booktabs}       
\usepackage{amsfonts}       
\usepackage{nicefrac}       
\usepackage{microtype}      
\usepackage{fancyhdr}       
\usepackage{graphicx}       
\graphicspath{{media/}}     

\usepackage{amsthm}
\usepackage{geometry}
\usepackage{amssymb}
\usepackage{amsmath}
\usepackage{verbatim}
\usepackage{color}
\usepackage{mathtools}

\newcommand{\rom}[1]{\lowercase\expandafter{\romannumeral #1\relax}}

\usepackage{algorithmicx}
\usepackage{algorithm}
\usepackage[noend]{algpseudocode}

\usepackage{multirow}
\usepackage{color}
\usepackage{amsmath}
\usepackage{verbatim}
\usepackage{soul}
\usepackage{lscape}
\usepackage{rotating}
\usepackage{epstopdf}
\usepackage{verbatim}
\usepackage{enumitem}
\usepackage{subcaption}
\usepackage[dvipsnames, svgnames]{xcolor}
\usepackage{listings}
\usepackage{hyperref}
\hypersetup{
    colorlinks=true,
    linkcolor=blue,
    filecolor=magenta,      
    urlcolor=cyan
}
\usepackage{tikz}
\usetikzlibrary{automata,positioning,shapes}
\usetikzlibrary{calc}

\newtheorem{thm}{Theorem}

\newtheorem{proposition}[thm]{Proposition}

\title{bsnsing: A decision tree induction method based on recursive optimal boolean rule composition}
\author{
  Yanchao Liu \\
  Department of Industrial \& Systems Engineering \\
  Wayne State University \\
  Detroit, Michigan, USA\\
  \texttt{yanchaoliu@wayne.edu} \\
}

\begin{document}
\maketitle

\begin{abstract}
This paper proposes a new mixed-integer programming (MIP) formulation to optimize split rule selection in the decision tree induction process,  and develops an efficient search algorithm that is able to solve practical instances of the MIP model faster than commercial solvers.  The formulation is novel for it directly maximizes the Gini reduction,  an effective split selection criterion which has never been modeled in a mathematical program for its nonconvexity.  The proposed approach differs from other optimal classification tree models in that it does not attempt to optimize the whole tree,  therefore the flexibility of the recursive partitioning scheme is retained and the optimization model is more amenable. The approach is implemented in an open-source R package named bsnsing.  Benchmarking experiments on 75 open data sets suggest that bsnsing trees are the most capable of discriminating new cases compared to trees trained by other decision tree codes including the rpart, C50, party and tree packages in R. Compared to other optimal decision tree packages, including DL8.5, OSDT, GOSDT and indirectly more, bsnsing stands out in its training speed, ease of use and broader applicability without losing in prediction accuracy. 
\end{abstract}


\section{Introduction} \label{sec:introduction}

Classification is the task of assigning objects to one of several predefined categories. 
A classification tree (or a decision tree classifier) is a predictive model represented in a tree-like structure.  Without making excessive assumptions about the data distribution, a classification tree partitions the input space into rectilinear (axis-parallel) regions and ultimately gives a set of \emph{If...Then...} rules to classify outputs or make predictions. 
Starting from the root node, each internal node is split into two or more child nodes based on the input values. The split stops when some terminal condition is met. The terminal nodes of the tree are called leaves, which represent the predicted target. Cases move down the tree along branches according to their input values and all cases reaching a particular leaf are assigned the same predicted value. 
The tree-like structure connects naturally to the divide-and-conquer strategy of how people judge, plan and decide.  Therefore, the technique is widely adopted for making decisions that bear substantial consequences for the decision maker.  Example applications include disease diagnosis \citep{10.1371/journal.pntd.0000196, GHIASI2020105400},  loan approval \citep{MANDALA2012406, 9325614} and investment selection \citep{Sorensen42}. 

In this paper,  we propose a classification tree induction method based on solving a mixed-integer programming (MIP) model at each split of a node.  This new method can generate trees that frequently outperform trees built by other off-the-shelf tree libraries in R,  the popular statistical computing system.  To achieve this,  the proposed MIP model explicitly maximizes the reduction in the node impurity and allows a tree node to be split by a multivariate boolean rule.  Such split rules are more flexible and in certain cases more efficient at characterizing nonlinear patterns in data,  and in the meantime,  remain highly interpretable.  To conquer the computational challenge,  we develop an efficient implicit enumeration algorithm that solves the MIP model faster than the state-of-the-art optimization solvers.  
Experiments on an extensive collection of machine learning data sets suggest that the method is accurate in prediction performance and scales reasonably well on large training sets.  The proposed framework is implemented in an open-source R package named {\ttfamily bsnsing},  available for a broad community of data science researchers and practitioners.  

While MIP techniques have been used in various ways in the recent literature for building classifiers,  our method is unique in several aspects.  First,  to our knowledge,  it is the first MIP model that is able to maximize the Gini reduction of a split,  which has been known to be an effective split criterion but in the meantime a nonlinear nonconvex function of the split decision.  Impurity reduction is an oft-used criterion for split selection in leading decision tree heuristics for its superiority in generating well-balanced child nodes over the accuracy maximization criterion,  but it has never been used in an optimization framework due to its nonconvexity.  Second, the MIP model is used for rule selection at each node split,  while other effective elements of the recursive partitioning framework,  such as split point generation,  early termination and tree pruning, can be separately implemented with great flexibility.  Third,  along with the novel formulation we also develop an efficient exact solution algorithm which runs faster than commercial solver codes,  making the {\ttfamily bsnsing} package independent of any commercial optimization solvers. 

The remainder of the paper is organized as follows.  Section \ref{sec:literature} reviews the literature on decision tree induction and particularly the recent literature on optimal classification tree (OCT) developments based on mixed-integer optimization. 
Section \ref{sec:method} develops the main models and algorithms that underlie the bsnsing package.   Section \ref{sec:evaluation} presents computational experiments to demonstrate the effectiveness of the proposed method and software tool.  Section \ref{sec:conclusion} concludes the paper with pointers for future work. 

\section{Related Literature} \label{sec:literature}
As an extremely flexible non-parametric framework, classification trees delegate a great deal of freedom to algorithm design and implementation \citep{IntroDM2005}. The entire search space for building the ``best'' tree can be enormous. Consider, for instance, splitting a node by a categorical variable consisting of 10 distinct levels. There are 115,974 non-trivial ways of splitting, i.e., $B_{10}-1$, the 10th \emph{Bell number} minus 1. Moreover, for a set of 10 single-variable split rules, there are more than 3.6 million (i.e., 10!) differnt ways to order them in a decision list. It is impractical to evaluate all possible splits and all possible ordering of rules. It is shown in \cite{HYAFIL197615} that the ``optimal decision tree'' problem is NP-complete, and this conclusion has been corroborated in many subsequent attempts at constructing optimal decision trees using various optimization modeling techniques. 

The difficulty incurred by the enormity of the search space has been dealt with along three routes in the literature. The first route is via using greedy splitting methods \citep{breiman1984book, Quinlan:1993:CPM:152181} under the recursive partitioning framework, in which a number of candidate splits are compared and a best one is chosen to split a node. In this general paradigm, there is a great variety of algorithms addressing issues such as how the split variables are selected, how the split points are determined and how the split quality is assessed, etc. Many efficient decision tree algorithms, including C4.5 \citep{Quinlan:1993:CPM:152181}, CHAID \citep{10.2307/2986296}, CART \citep{breiman1984book}, GUIDE \citep{loh2009}, and the recent Bayesian-based approach \citep{letham2015, PMLR2017_ScalableBRL} fall under this paradigm. 

The second route is to trim the overall search space down to a reduced model space (as a surrogate) in which global optimization is used to find an optimal model.  
A popular choice of the surrogate space is the frequent itemsets, e.g., results from association rule mining algorithms \citep{Agrawal93miningassociation,WIDM:WIDM1074,Liu1998}.  \cite{Bertsimas12aninteger} devised a two-step approach, in which the first step is to generate an efficient frontier of $L$ candidate itemsets by solving $L$ mixed-integer optimization (MIO) problems, one for each candidate itemset, and then solve another (larger) MIO problem to rank all the candidates based on their predictive accuracy on all transactions. The top-ranked candidate is chosen as the final classifier.  This approach is computationally demanding because of the attempt to build the whole classifier by solving one large MIO problem. 
\cite{Nijssen:2010:OCD:1830978.1830979} noted the link between decision trees over a binary feature space and the itemset lattice, and built a recursive tree learning algorithm,  which did not invoke a numerical optimization process.  \cite{10.1145/3097983.3098047} considered the class of rule lists assembled from pre-mined frequent itemsets and searched for an optimal rule list that minimizes a regularized risk function, 
which was able to solve (and prove optimality for) fairly large classification instances. 
This stream of research is important in constructing the notion of optimality in tree learning and exploring the use of discrete optimization techniques such as the branch-and-bound algorithm. 

The third route is using exhaustive search for comprehensible rules that do not involve too many clauses.
The 1R algorithm \citep{Holte1993} searched exhaustively the space of single-variable rules and then made classification or prediction based on the best rule found. Considering its simplicity, it was a surprise that it performed well on many data sets.  Nonetheless, the single-variable, single-split strategy will apparently sacrifice performance in cases where more complex and subtle patterns need to be characterized. The EXPLORE algorithm \citep{Rijnbeek:2010:FSA:1825171.1825175} performs an exhaustive search in the complete rule spaces consisting of 1 term, 2 terms and so forth, until the increment of the number of allowed terms stops giving a better performance on the validation set than the previous iteration. By searching all possible Disjunctive Normal Forms (DNF), the algorithm can find the best DNF rule up to a certain complexity level. An important observation given by this paper is that the phenomenon of oversearching \citep{Quinlan:1995:OLS:1643031.1643032}, i.e., the hypothesis that the more rules are evaluated the greater the chance of finding a fluke and poorly-generalizable rule, does not always hold. Even though exhaustive search is hardly viable for large cases, this gives an encouraging indication that aggressively seeking optimality on the training set does not necessarily incur overfitting, if the sense of optimality is defined on a prudent metric and the tree complexity is properly regulated,  an insight that was also presented in \cite{Bertsimas2017}.  In this paper, we develop a more principled exhaustive search procedure to solve the split selection problem.

\begin{table}[h!]
\begin{center}
\small
\caption{Comparison of different tree structures.  The proposed multi-variable rectilinear splits are more efficient than single-variable splits in carving out nonlinear features, and preserves interpretability better than linear combination splits. }
\label{tbl:tree_structure_compare}
\begin{tabular}{| c | p{6.5cm} | p{4cm} |}
     \hline
      Illustration & Characteristics & Software Codes\\ \hline
      \begin{minipage}{.2\textwidth}
      \centering \includegraphics[width=0.7\textwidth]{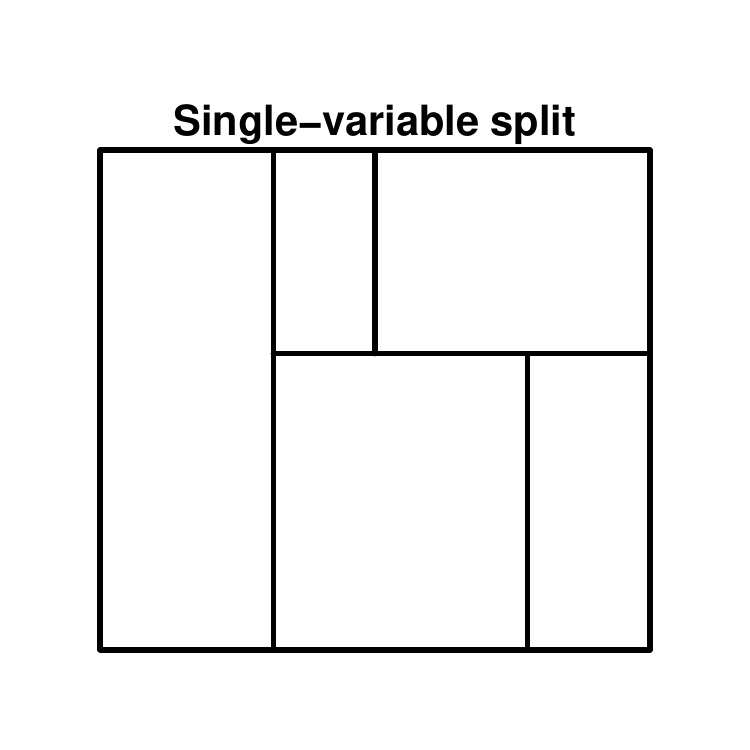}
      \end{minipage}
      & 
      \vspace{0.15 in}
      \begin{itemize}[leftmargin=*, noitemsep]
      \setlength \itemsep{0em}
      \item Intuitive decision rules
      \item E.g., $\{\text{Is } \texttt{Age} \le 20?\}$ 
      \item Split only produces $(n-1)$-dimensional rectilinear halfspaces
      \end{itemize}
      &
      \begin{itemize}[leftmargin=*, label={}, noitemsep]
      \setlength \itemsep{0em}
      \item ID3, C5.0, CHAID
      \item SLIQ, rpart, party
      \end{itemize}
      \\ 
      \cline{2-3}
      \begin{minipage}{.2\textwidth}
      \centering \includegraphics[width=0.7\textwidth]{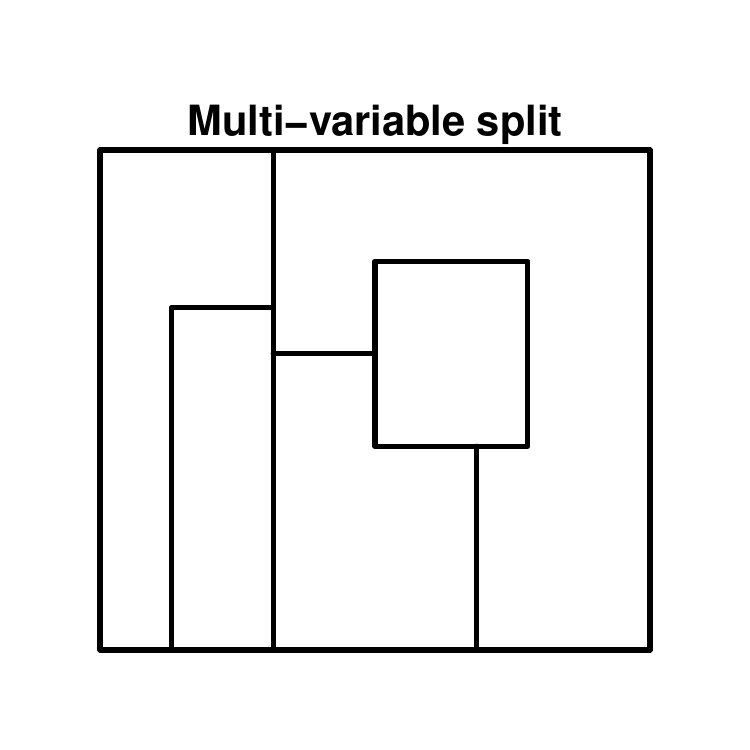}
      \end{minipage}
      & 
      \vspace{0.15 in}
      \begin{itemize}[leftmargin=*, noitemsep]
      \setlength \itemsep{0em}
      \item Intuitive decision rules
      \item E.g., $\{\text{Is } 20 \le \texttt{Age} \le 25 \; \& \;  19 \le \texttt{BMI} \le 24?\}$       
      \item Split can generates closed and open hypercubes of any dimension
      \end{itemize}
	  &
      \begin{itemize}[leftmargin=*, label={}, noitemsep]
      \setlength \itemsep{0em}
      \item CORELS,  \textbf{bsnsing}
      \end{itemize}
      \\ 
      \cline{2-3}
      \begin{minipage}{.2\textwidth}
      \centering \includegraphics[width=0.7\textwidth]{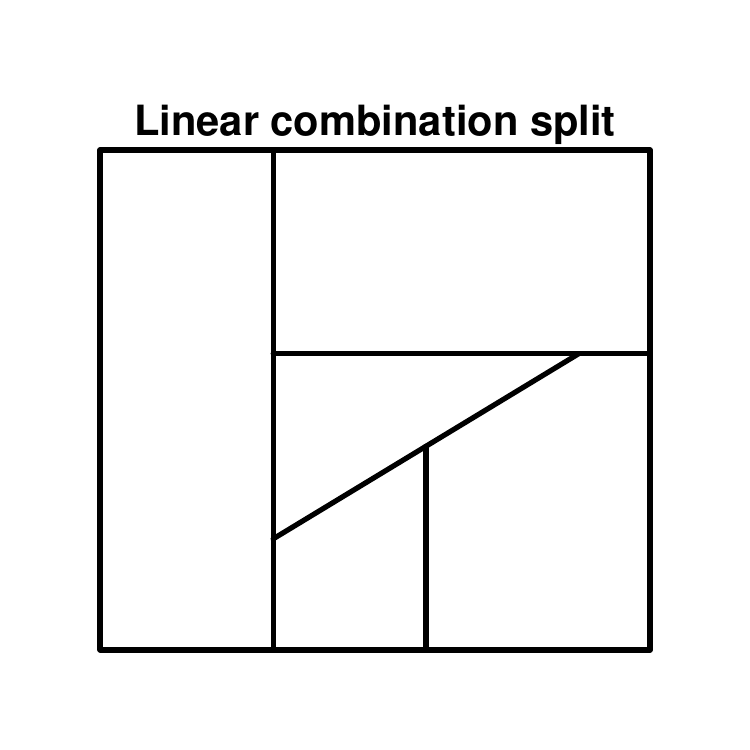}
      \end{minipage}
      & 
      \vspace{0.15 in}
      \begin{itemize}[leftmargin=*, noitemsep]
      \setlength \itemsep{0em}
      \item Obscure decision rules
      \item E.g., $\{\text{Is } 0.3* \texttt{Age} - 0.5* \texttt{BMI} \ge 3.3?\}$
      \item Split can produce halfspaces of any dimension
      \end{itemize}
      & 
      \begin{itemize}[leftmargin=*, label={}, noitemsep]
      \setlength \itemsep{0em}
      \item CART, FACT
      \item QUEST,  CRUISE
      \item GUIDE, OCT-H
      \end{itemize}
      \\ 
      \hline
      \begin{minipage}{.2\textwidth}
      \centering \includegraphics[width=0.7\textwidth]{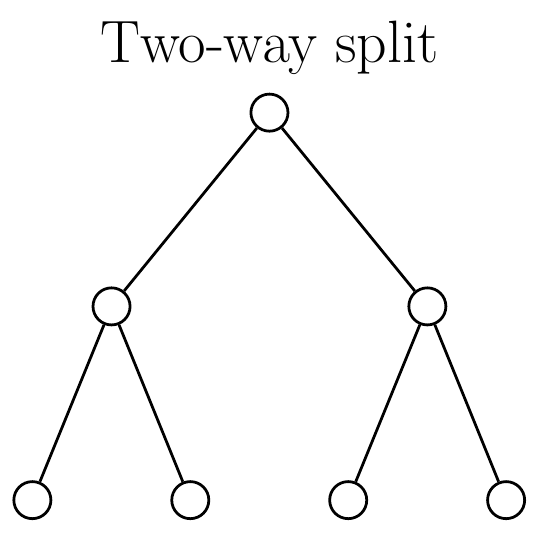}
      \end{minipage}
      & 
      \vspace{0.15 in}
      \begin{itemize}[leftmargin=*, noitemsep]
      \setlength \itemsep{0em}
      \item \emph{If...Then...} rule clause
      \item Less efficient for multi-class problem
      \end{itemize}
      & 
      \begin{itemize}[leftmargin=*, label={}, noitemsep]
      \setlength \itemsep{0em}
      \item CART, QUEST, SLIQ
      \item rpart, \textbf{bsnsing}
      \end{itemize}
      \\ 
      \cline{2-3}
      \begin{minipage}{.2\textwidth}
      \centering \includegraphics[width=0.98\textwidth]{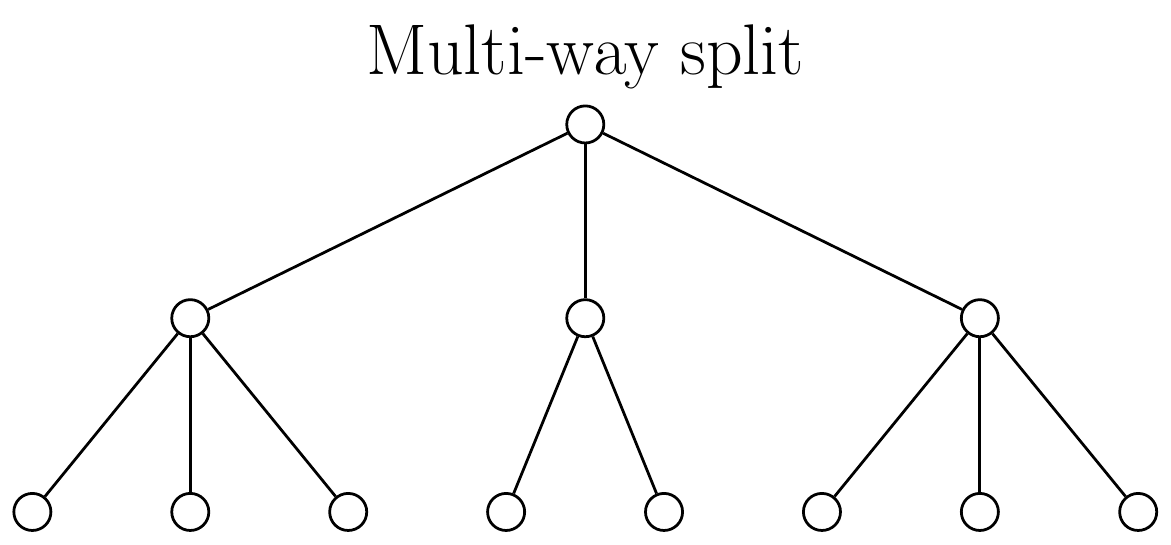}
      \end{minipage}
      & 
      \vspace{0.15 in}
      \begin{itemize}[leftmargin=*, noitemsep]
      \setlength \itemsep{0em}
      \item \emph{If...Then...ElseIf...} rule clause
      \item Suits multi-class problem better
      \item Depletes data too quickly
      \end{itemize}
      & 
      \vspace{0.15in}
      \begin{itemize}[leftmargin=*, label={}, noitemsep]
      \setlength \itemsep{0em}
      \item C5.0, CHAID, FACT
      \item CRUISE, QUEST
      \item party
      \end{itemize}
      \\
      \hline
      
\end{tabular}
\end{center}
\end{table}

Using mixed-integer optimization to tackle classification problems has been frequently investigated in recent years,  following the seminar work of \cite{Bertsimas12aninteger}. 
\cite{pmlr-v28-malioutov13} formulated the rule-based classification problem as an integer program and showed that under certain conditions (among which,  an excessively large pool of boolean questions needed to be generated from the original feature set) the rules could be recovered exactly by solving the linear programming (LP) relaxation.  A probabilistic guarantee of recovery was shown when the required conditions were satisfied weakly.   
\cite{10.1145/2623330.2623648} addressed the problem of class imbalance in classification training and proposed an MIP model to find the classification rule set that optimizes a weighted balance between positive and negative class accuracies.   They developed a ``characterize then discriminate'' approach to decompose the problem into manageable subproblems hence alleviate the computational challenge of solving the full MIP.   
\cite{Bertsimas2017} cast the problem that CART attempted to solve as a global optimization problem, and instantiated the canonical problem with two MIP models,  OCT for building trees using univariate splits and OCT-H for trees of multivariate splits (separating hyperplanes). The models minimized a weighted sum of the total misclassification cost and the number of splits in the tree,  and were constrained by two hyperparameters, the tree depth and the leaf node size.  The weighting factor in the objective function needed to be tuned via a validation set to achieve the best performance.  The robust version of these models were given in \cite{doi:10.1287/ijoo.2018.0001}.  The strengths of OCT and OCT-H complemented each other and they were able to outperform CART in many cases by significant margins. 
\cite{JMLR:v18:16-003} focused on searching for a small number of short rules (disjunctive of conjunctives, e.g.,  ``(A and B) or C or ... '' kind of a rule) by approximately solving a ``maximum a posteriori'' problem by the simulated annealing algorithm.  The authors applied the method to predict user behavior in a recommender system and reported favorable performance.  
\cite{Zhang_2019} formulated the optimal classification tree (OCT) problem as an integer program in which the number of integer (binary) decision variables would not depend on the number of training data points (though the number of constraints would, and big-M constraints were used).  The formulation was demonstrated to outperform previous OCT formulations,  including \cite{10.1007/978-3-319-59776-8_8} and \cite{Bertsimas2017},  on several test data sets.

\cite{DBLP:journals/corr/abs-1904-12847} proposed an optimal sparse decision trees (OSDT) algorithm that extends the CORELS algorithm \citep{10.1145/3097983.3098047} (which creates optimal rule lists) to create optimal trees. The algorithm attempts to minimize the weighted sum of the misclassification error and the number of leaves in the tree. A specialized search procedure within a branch-and-bound framework is employed for solution. A Python program that implements OSDT is available on Hu's Github page. 
\cite{lin2020generalized} provided a general framework for decision tree optimization which was able to handle a variety of objective functions and optimize over continuous features. The authors observed orders of magnitude speedup in decision tree construction compared to the state-of-the-art. A C++ based implementation, called GOSDT, as well as a Python wrapper, is available on Lin's Github page. 
\cite{Aglin_Nijssen_Schaus_2020} developed a DL8.5 algorithm which extends DL8 initially proposed in \cite{10.1145/1281192.1281250}. DL8.5 draws upon the association rule mining literature \citep{Agrawal93miningassociation} and uses branch-and-bound search along with a caching mechanism to achieve a fast training speed. It has been demonstrated that DL8.5 outruns the BinOCT method by orders of magnitude in training speed. DL8.5 was implemented in C++ and a Python package is publicly available (i.e., pip install dl8.5). 
In Section \ref{sec:comp_dl85_osdt}, we perform computational comparisons with DL8.5, OSDT and GOSDT on a number of binary classification data sets to demonstrate {\ttfamily bsnsing}'s advantage in training speed amongst these latest developments in the OCT literature. 
\cite{DBLP:conf/nips/ZhuMPNK20} proposed an MIP model for supervised classification by optimally organizing support vector machine (SVM) type of separating hyperplanes in a tree structure of a given depth,  and achieved outstanding performance on a collection of test sets.  
An earlier work of this kind can be found in \cite{doi:10.1287/ijoc.1030.0047}, where the separating hyperplane was obtained via solving a nonlinear conconvex program. 
Ease of interpretation (or interpretability) of such tree-like models was clearly not an emphasis in those works. 
\cite{aghaei2020learning} presented a flow-based MIP formulation for the OCT problem. The formulation  did not use big-M constraints and hence boasted a stronger LP relaxation than alternative formulations.  The authors developed a Bender's decomposition paradigm to further improve solution efficiency.  Substantial speedup in comparison to other OCT approaches was reported. 

Existing MIP-based OCT investigations invariably attempted to internalize (i.e.,  globally optimize) the whole process of building the (tree or rule-based) classifier,  and abandoned the recursive partitioning framework.  
Consequently, any post-training modification, such as pruning,  to the optimal tree would nullify optimality in intractable ways,  so it would be difficult to justify pruning or other tactics aimed at improving prediction performance.  Furthermore,  making all decisions regarding tree induction in a single MIP model is inevitably challenged by the dilemma that, either a lot of simplifications must be imposed to the classifier to limit the size of the search space (hence introducing high bias), or the decision model ends up being computationally prohibitive.  In the author's opinion, solving an MIP usually takes too much time to warrant its inclusion in a meaningful hyperparameter search process which requires training a number of candidate classifiers.  For instance,  a 10-minute run time is not uncommon for solving a moderately sized MIP model,  but it would feel somewhat long for a user of software tools such as R, SAS,  Stata and IBM SPSS, which are able to produce a tree in no more than a couple of seconds (for reasonably sized data sets, such as those oft-used for benchmarking in the literature).  As a result, software codes for most OCT approaches, despite their potential appeal in high stake applications where training/tuning time is less of a concern than other merits such as interpretability, accuracy and rule set sparsity, etc., have not been widely picked up by the broader statistical learning or data science community.  
The {\ttfamily bsnsing} approach attempts to alleviate some of these challenges by taking a middle ground - using MIP to optimize only the split decision while other aspects of tree management,  such as when to stop,  how to generalize candidate split points,  and pruning,  etc.,  are offloaded to the recursive partitioning framework. 

Table \ref{tbl:tree_structure_compare} summarizes the prevalent tree structures available in software tools and how the proposed {\ttfamily bsnsing} package expands the landscape.   Existing algorithms under the recursive partitioning framework either employ one-variable splits which are too restrictive for expressing nonlinear patterns, or employ linear combination splits which obscures interpretation.  \cite{breiman1984book} briefly entertained the idea of constructing Boolean combinations of single-variable splits and recognized the associated difficulty in split search because the combinations are too many to enumerate.  To our knowledge, the {\ttfamily bsnsing} package is the first decision tree package to implement Boolean combinations of splits in which the combinatorial difficulty is mitigated by an efficient search algorithm.


\section{Models and Methods} \label{sec:method}

\subsection{Preliminaries and the Framework Overview}
Let $\mathcal{X}$ denote the input space containing all possible input vectors $x$. Suppose that objects characterized by $x$ fall into $J$ classes and let $C$ be the set of classes, i.e., $C = \{1, \ldots, J\}$. A classifier is a rule that assigns a class membership in $C$ to every vector in $\mathcal{X}$. In other words, a classifier is a partition of $\mathcal{X}$ into $J$ disjoint subsets $A_1, \ldots, A_j$, $\mathcal{X} = \cup_j A_j$, such that for every $x \in A_j$ the predicted class is $j$. A classifier is constructed or trained by a \emph{learning sample} $\mathcal{L}$, which consists of input data on $n$ cases together with their actual class membership, i.e., $\mathcal{L} = \{(x_1, j_1), \ldots, (x_n, j_n)\}$, where $x_i \in \mathcal{X}$ and $j_i \in \{1, \ldots, J\}$, for $i = 1, \ldots, n$. At present, we consider the binary classification problem where $J = 2$ and call the two classes as being \emph{positive} and \emph{negative}, respectively, i.e., $C = \{\text{Positive}, \text{ Negative}\}$.

\begin{figure}[h!]
\begin{center}
\includegraphics[width=\linewidth]{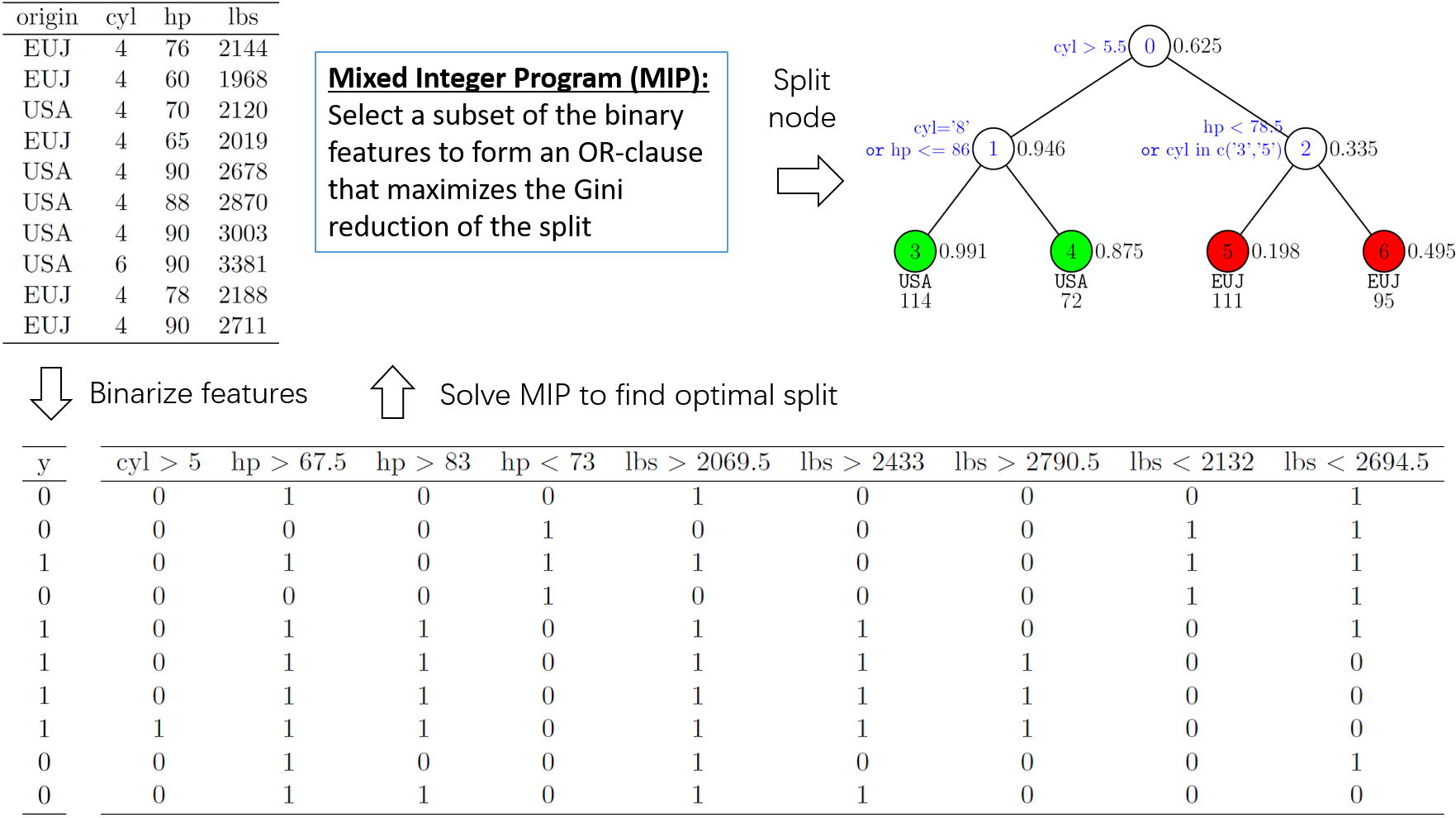}
\caption{Process flow of the {\ttfamily bsnsing} method.  First, a binary feature matrix $B$ is created based on the original input variables. Each feature represents a general question that demands a yes/no answer.  Next, a mixed-integer program (MIP) is solved to select the set of questions to form a boolean OR-clause that would maximize impurity reduction.  Finally, the tree node is split by the selected rule(s).  Detailed annotation of the {\ttfamily bsnsing} tree plot is given in the appendix.}
\label{fig:bsnsing_split_process}
\end{center}
\end{figure}

Given a learning sample $\mathcal{L}$ available at a tree node,  the algorithm takes two steps to split the node. First, all input variables are coded into binary features.  Each binary feature represents a question that demands a Yes/No answer based on the value of the original variable. We call this process \emph{binarization} of the input space.  
The outcome of this step includes (1) an $n$-by-$m$ matrix $B$ consisting of 0/1 entries, where $m$ is the total number of binary features created in the process,  and (2) the original binary response vector $y$ in the sample $\mathcal{L}$. 

The second step determines a boolean OR-clause to split the node.   Here,  a boolean OR-clause refers to a set of general questions joined by the logical OR operator,  e.g., $\{$Is Age $>$ 35 or Age $\le 28$  or BMI $\ge$ 30?$\}$.  If a case answers yes to any question in the clause, it will be classified as a positive case; otherwise, it will be classified as a negative case.  The selection of questions into the clause is the decision to be made here,  which will be formulated as a combinatorial optimization problem via mixed-integer programming.  Figure \ref{fig:bsnsing_split_process} demonstrates how {\ttfamily bsnsing} handles these steps.  

\subsection{Feature Binarization}
We very briefly outline the default feature binarization approach implemented in {\ttfamily bsnsing}.  It is not the emphasis of this paper and it is extensible by other developers. 
For a numeric variable, samples are sorted based on its value. If the two classes are perfectly separable, i.e., the minimum value of one class is greater than the maximum value of the other class, the split point is returned and both child nodes are marked as a leaf. Otherwise, the sorted list of samples is scanned twice, in the sequential and reverse order, to find ``greater-than'' and ``less-than'' type of split conditions, respectively. In this manner, the algorithm implements the subsumption principle as described in \cite{Rijnbeek:2010:FSA:1825171.1825175} to ensure that only potentially valuable split conditions are generated and ones that are not in the efficient frontier are left out.  Other approaches for binarizing numeric features, such as using the empirical quantiles as cut points, see, e.g., \cite{pmlr-v28-malioutov13}, are also worth implementing in future work. For a categorical variable of $L$ unique levels, the binarization process creates $L$ binary dummy variables when $L$ is below a threshold (default 30).  When $L$ is greater than the threshold,  value grouping is applied before creating dummy variables.  Finally,  binary features generated by other decision tree packages can be imported to {\ttfamily bsnsing} for optimal selection as well. 

\subsection{Mixed-Integer Program to Maximize Gini Reduction}
The Gini index,  developed by Italian statistician Corrado Gini in 1912,  is a measure of variability for categorical data.  It can be used to measure the impurity of a decision tree node.  In general,  the Gini index for a set of objects (e.g., cases contained in a tree node) that come from $J$ possible classes is given by $1-\sum_{j = 1}^J p_j^2$,  where $p_j$ is the relative frequency of the target class $j$,   $ j\in\{1,\ldots,J\}$,  in the node.  A split that produces a large reduction in Gini (i.e., $\Delta$Gini) is preferred.  The $\Delta$Gini splitting criterion was first proposed in \cite{breiman1984book}.  It is available in many decision tree codes such as  {\ttfamily rpart} and {\ttfamily tree} packages in {\ttfamily R} and the decision tree method in SAS Enterprise Miner.   
We will formulate the $\Delta$Gini criterion as a mixed-integer linear program of the splitting decision.  To our knowledge,  no prior work has done so. 

Let $P$ and $N$ denote the numbers of positive and negative cases,  respectively,  at the current (parent) node.  The Gini index of this node,  denoted by $G(\text{parent})$, is
\begin{align}
& G(\text{parent}) = 1-\left(\frac{P}{P+N}\right)^2 - \left(\frac{N}{P+N}\right)^2 = \frac{2P\cdot N}{(P+N)^2}
\end{align}
Suppose the node is split into two child nodes by a split rule,  that is,  cases that are ruled to be positive fall in the left node and cases that are ruled to be negative fall in the right node.  Let $TP$ and $FP$ denote the numbers of positive and negative cases,  respectively,  that fall in the left node,  and $TN$ and $FN$ denote the numbers of negative and positive cases,  respectively,  that fall in the right node,  then the Gini indexes of the left and right nodes are
\begin{align*}
& G(\text{left}) = \frac{2 \cdot TP \cdot FP}{(TP+FP)^2} \text{ and } G(\text{right}) = \frac{2 \cdot TN \cdot FN}{(TN+FN)^2}
\end{align*}

The $\Delta$Gini of the split is defined to be the Gini index of the parent node minus the weighted sum of the Gini indexes of the child nodes, whereas the weights are the squared proportions of cases that fall in each child node,  
\begin{align}
& \Delta G = G(\text{parent}) - \left(\frac{TP + FP}{P + N}\right)^2 \cdot G(\text{left}) - \left(\frac{TN + FN}{P + N}\right)^2 \cdot G(\text{right})  \label{eq:deltaGini}
\end{align}
Note that the squared proportions are used here, in stead of the proportions as originally proposed in \cite{breiman1984book}, for the ease of mathematical modeling. 
Note that the value of $\Delta G$ is not affected by the classification labels assigned to the child nodes.  For example,  if we were to classify the left node as negative and the right node as positive,  it would only cause a swap between $TP$ and $FN$ and a swap between $FP$ and $TN$ in the above equations,  which would not affect the value of $\Delta G$.  
Equation (\ref{eq:deltaGini}) can be simplified to $\Delta G =(2P\cdot N - 2 (TP \cdot FP + TN \cdot FN))/(P + N)^2$.  Since $P$ and $N$ are known values for the parent node irrespective of the split,  maximizing $\Delta G$ is equivalent to minimizing $TP \cdot FP + TN \cdot FN$,  which is in turn equivalent to minimizing 
\begin{align}
P \cdot FP + N \cdot FN - 2\cdot FN \cdot FP \label{eq:Gini_obj}
\end{align}
Here,  $FP$ and $FN$ are variables whose values depend on the split rule. 

\begin{table}
\centering 
\caption{Notation definition for the mathematical program.} \label{tb:mip_sym_def}
\begin{tabular}{l l}
\hline
Symbol & Meaning \\
\hline
$\mathcal{I}$ & $=\{1,\ldots, n\}$,  index set of cases \\
$\mathcal{P},  \, \mathcal{N}$ & index sets of Positive and Negative cases, respectively \\
$\mathcal{K}$ & $=\{1, \ldots, m\}$,  index set of questions \\
\hline
$B_{ik}$ & $=1$ if case $i$ answers Yes for question $k$,  0 otherwise \\
$y_i$ & $=1$ if case $i$ is positive; 0 otherwise \\
\hline
$w_k$ & $=1$ if question $k$ is selected into the split rule; 0 otherwise \\
$z_i$ & $=1$ if case $i$ is classified as Positive; 0 Negative \\
$\theta_{ij}$ & a binary variable for each pair of cases $i$ and $j$ \\
\hline
\end{tabular}
\end{table}

Mathematical symbols used in the MIP formulation is defined in Table \ref{tb:mip_sym_def}. 
Let the binary variable $w_k$ indicate whether question $k$ is selected (=1) or not (=0),  then the product $B_{ik} w_k$ equals 1 if both question $k$ is selected and case $i$ answers yes for the question.  When the selected questions form an OR-clause classification rule,  the case $i$ is classified as Positive ($z_i = 1$) if there exists a $k \in \mathcal{K}$ such that $B_{ik} w_k = 1$. Conversely,  case $i$ is classified as Negative ($z_i = 0$) if $B_{ik} w_k = 0$ for all $k \in \mathcal{K}$. 
This classification rule is expressed by the following linear constraints. 
\begin{align}
& & B_{ik} w_k \le z_i,  & & \forall i \in \mathcal{I},  \, \forall k \in \mathcal{K} \label{eq:wz1} \\
& & \sum_{k \in \mathcal{K}} B_{ik} w_k \ge z_i,  & & \forall i \in \mathcal{I} \label{eq:wz2}
\end{align}
We only need to enforce $0 \le z_i \le 1$,  for $i \in \mathcal{I}$ in the formulation,  since the integrality of variable $z_i$ is implied by the constraints (\ref{eq:wz1}) and (\ref{eq:wz2}), the fact that each $w_k$ is binary and the sense (i.e., minimization) of the optimization objective. 

Using variable $z_i$,  we can express the terms in (\ref{eq:Gini_obj}) as follows.
\begin{align*}
& P \cdot FP = P \cdot \sum_{j \in \mathcal{I}} (1 - y_j) z_j  = \sum_{i \in \mathcal{P}} \sum_{j \in \mathcal{N}} z_j \\
& N \cdot FN = N \cdot \sum_{i \in \mathcal{I}} y_i (1-z_i) = \sum_{j \in \mathcal{N}} \sum_{i \in \mathcal{P}} (1-z_i) \\
\begin{split}
FN \cdot FP & = \left(\sum_{i \in \mathcal{I}} y_i (1-z_i)\right) \left( \sum_{j \in \mathcal{I}} (1 - y_j) z_j \right) \\
& = \sum_{i, j \in \mathcal{I}} y_i (1-z_i) (1-y_j) z_j \\ 
& = \sum_{i \in \mathcal{P}} \sum_{j \in \mathcal{N}} (1-z_i) z_j 
\end{split}
\end{align*}
Minimizing (\ref{eq:Gini_obj}) would reward a greater value of $FN \cdot FP$,  therefore we can replace each term $(1-z_i) z_j$ in the last equation by a free variable $\theta_{ij}$,  defined for $i \in \mathcal{P}$ and $j \in \mathcal{N}$,  and impose the following constraints,
\begin{align}
& & \theta_{ij} \le 1-z_i,  & & \forall i \in \mathcal{P},\, \forall j \in \mathcal{N} \label{eq:theta1} \\
& & \theta_{ij} \le z_j,  & & \forall i \in \mathcal{P},\, \forall j \in \mathcal{N} \label{eq:theta2} 
\end{align}


Putting everything together,  the problem that optimizes Gini reduction is formulated as a mixed-integer program as follows.  Let us call it OPT-G.  

\noindent (OPT-G):
\begin{align}
& & & \text{Minimize} & & \sum_{i \in \mathcal{P}} \sum_{j \in \mathcal{N}} (1-z_i + z_j - 2\theta_{ij}) \label{eq:mip_obj} \\
& & & \text{s. t.} & & \text{(\ref{eq:wz1}), (\ref{eq:wz2}),  (\ref{eq:theta1}) and (\ref{eq:theta2})} \notag \\
& & &  & & w_k \in \{0, 1\},  & & \forall k\in \mathcal{K} \label{eq:mip_w}\\
& & & & & 0 \le z_i \le 1,  & & \forall i \in \mathcal{I} \label{eq:mip_z} \\
& & & & & \theta_{ij} \text{ free} & & \forall i \in \mathcal{P},  \, j \in \mathcal{N} \label{eq:mip_theta}
\end{align} 
For each pair of cases $i \in \mathcal{P}$ and $j \in \mathcal{N}$,  $\theta_{ij}$ equals 1 only when both cases are classified incorrectly.  When this happens,  the corresponding objective term $(1-z_i + z_j - 2\theta_{ij})$ will be the same value (zero) as when both cases are classified correctly.  This corroborates the theory that the $\Delta G$ split criterion is agnostic of the polarity of the labels assigned to the child nodes.  
Equation (\ref{eq:mip_z}) can be further reduced to $z_i \le 1$ for $i \in \mathcal{P}$ and $z_i \ge 0$ for $i \in \mathcal{N}$.  
Most MIP models in the OCT literature attempted to minimize misclassification (plus other terms to suppress overfitting).  The OPT-G model can be reduced to minimizing misclassification simply by removing the objective terms and constraints that involve $\theta$.  Specifically,  the following model (OPT-E) explicitly minimizes the number of misclassified cases from the split. 

\noindent (OPT-E):
\begin{align}
& & & \text{Minimize} & & \sum_{i \in \mathcal{P}} \sum_{j \in \mathcal{N}} (1-z_i + z_j) \label{eq:mip_obj_e} \\
& & & \text{s. t.} & & \text{(\ref{eq:wz1}), (\ref{eq:wz2}),  (\ref{eq:mip_w}) and (\ref{eq:mip_z})} \notag 
\end{align} 

The Gini-based splitting criterion is designed to purify the class composition in the child nodes.  Compared to error-based criterion,  it is conducive to more balanced child nodes.  This is demonstrated in the simple example in Figure \ref{fig:compare_gini_vs_error}.  Therefore,  in this paper we focus on analyzing the OPT-G model and its solution strategy.  Splits based on the OPT-E model is available in the {\ttfamily bsnsing} package through the option {\ttfamily opt.model = `error'}.  Users can choose Gurobi,  CPLEX and lpSolve to solve the MIP model by setting the {\ttfamily opt.solver} option in {\ttfamily bsnsing},  provided that the chosen solver and its R API package are installed.

\begin{figure}
\centering
    \begin{tikzpicture}
        \node[state, ellipse, minimum width=2cm, minimum height=1.3cm, align=center] at (2,3)            (a) {\large $\oplus$$\oplus$$\oplus$ \\\large $\ominus$$\ominus$$\ominus$};
        \node[state, ellipse, minimum width=2cm, minimum height=1.3cm, align=center] at (0,0)            (a1) {\large $\oplus$};
        \node[state, ellipse, minimum width=2cm, minimum height=1.3cm, align=center] at (4,0)            (a2) {\large $\oplus$$\oplus$ \\ \large$\ominus$$\ominus$$\ominus$};
        \node[state, ellipse, minimum width=2cm, minimum height=1.3cm, align=center] at (9,3)            (b) {\large $\oplus$$\oplus$$\oplus$ \\ \large $\ominus$$\ominus$$\ominus$};
        \node[state, ellipse, minimum width=2cm, minimum height=1.3cm, align=center] at (7,0)            (b1) {\large $\oplus$$\oplus$$\ominus$};
        \node[state, ellipse, minimum width=2cm, minimum height=1.3cm, align=center] at (11,0)            (b2) {\large $\ominus$$\ominus$$\oplus$};

        \draw[every loop]
            (a) edge[auto=right] node {\large +} (a1)
            (a) edge[auto=left] node {\large -} (a2)
            (b) edge[auto=right] node {\large +} (b1)
            (b) edge[auto=left] node {\large -} (b2);
    \end{tikzpicture}
\caption{The two ways of splitting the parent node gives the same number of misclassifications,  hence they are indistinguishable under OPT-E.  In contrast,  OPT-G favors the one leading to more balanced child nodes. } \label{fig:compare_gini_vs_error}
\end{figure}
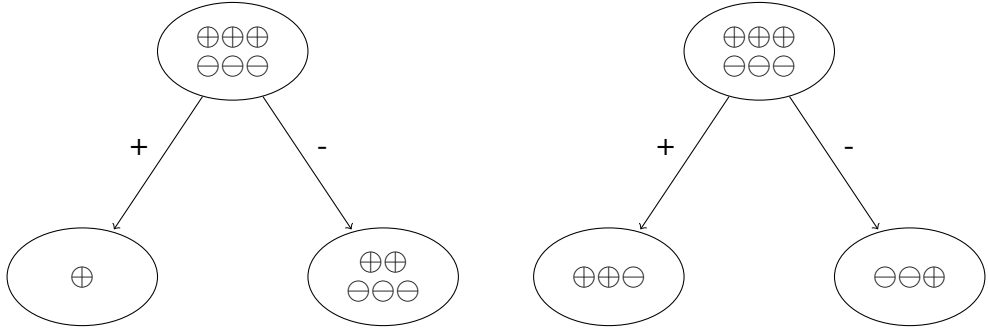



In the tree-building process, the instance size of the OPT-G model is largest at the root node,  and decreases exponentially as the process moves down the tree branches.  For large training data sets,  the root node model can become computationally prohibitive for even the commercial solvers such as CPLEX and Gurobi.  In the next section, we propose an implicit enumeration approach for solving OPT-G that scales better than commercial solver codes in practical settings.

\subsection{Implicit enumeration (ENUM) algorithm for solving OPT-G}
Each candidate solution to OPT-G can be denoted by its corresponding index set $\mathcal{S}$ of the selected questions,  i.e.,  $\mathcal{S} := \{k \in \mathcal{K} \mid w_k = 1\}$.   At a solution $\mathcal{S}$,  let us denote the false positive and false negative counts at the solution by $FP_\mathcal{S}$ and $FN_\mathcal{S}$, respectively,  
and define $\nu(\mathcal{S})$ to be the objective value at the solution,  that is,
\begin{align}
& \nu(\mathcal{S}) := P \cdot FP_{\mathcal{S}} + N \cdot FN_{\mathcal{S}} -2\cdot FP_{\mathcal{S}} \cdot FN_{\mathcal{S}}  \label{eq:def_v} 
\end{align}

We know that starting at any $\mathcal{S}$,  selecting more questions into the solution (i.e., enlarging $\mathcal{S}$)  would only encourage extra cases to fall in the left node (i.e., answer ``yes'' to the OR-clause), and hence, would cause $FP$ to either stay at the same value or increase, and would cause $FN$ to stay at the same value or decrease. 
In other words,  for any $\mathcal{S}_+ \supset \mathcal{S}$,  we have $FP_{\mathcal{S}_+} \ge FP_{\mathcal{S}}$ and $FN_{\mathcal{S}_+} \le FN_{\mathcal{S}}$.
Leveraging this property,  we can derive a lower bound on the objective value for any possible solution that ``branches out'' from a given solution $\mathcal{S}$.  

\begin{proposition} \label{prop1}
For any superset of $\mathcal{S}$,  denoted by $\mathcal{S}_+$,  the following inequalities hold.
\begin{align*}
\nu(\mathcal{S}_+) \ge \begin{cases}
\nu(\mathcal{S}),  & \text{if } FN_\mathcal{S} < P/2 \text{ and } FP_\mathcal{S} > N/2 \\
P \cdot FP_{\mathcal{S}},  & \text{if } FN_\mathcal{S} < P/2 \text{ and } FP_\mathcal{S} \le N/2 \\
N \cdot (P - FN_{\mathcal{S}}),  & \text{if } FN_\mathcal{S} \ge P/2 \text{ and } FP_\mathcal{S} > N/2 \\
0, & \text{if } FN_\mathcal{S} \ge P/2 \text{ and } FP_\mathcal{S} \le N/2 
\end{cases}
\end{align*}
\end{proposition}

\begin{proof}
In the case $FN_\mathcal{S} < P/2$ and $FP_\mathcal{S} > N/2$,  we have $N - 2\cdot FP_{\mathcal{S}_+} \le N - 2\cdot FP_{\mathcal{S}} < 0$ and $P - 2\cdot FN_{\mathcal{S}_+} \ge P - 2\cdot FN_{\mathcal{S}} > 0$,  therefore,
\begin{align*}
\begin{split} \nu(\mathcal{S}_+)  & = P \cdot FP_{\mathcal{S}_+} + (N-2\cdot FP_{\mathcal{S}_+}) \cdot FN_{\mathcal{S}_+} \\
& \ge P \cdot FP_{\mathcal{S}_+} + (N-2\cdot FP_{\mathcal{S}_+}) \cdot FN_{\mathcal{S}} \\
& = (P - 2\cdot FN_{\mathcal{S}_+}) \cdot FP_{\mathcal{S}_+} + N \cdot FN_{\mathcal{S}} \\
& \ge (P - 2\cdot FN_{\mathcal{S}}) \cdot FP_{\mathcal{S}} + N \cdot FN_{\mathcal{S}} \\
& = \nu(\mathcal{S})
\end{split}
\end{align*}
In the case $FN_\mathcal{S} < P/2 \text{ and } FP_\mathcal{S} \le N/2$,  we have $P - 2\cdot FN_{\mathcal{S}_+} \ge P - 2\cdot FN_{\mathcal{S}} > 0$ and $N - 2\cdot FP_{\mathcal{S}} \ge 0$,  therefore,
\begin{align*}
\begin{split} \nu(\mathcal{S}_+) & = N \cdot FN_{\mathcal{S}_+} + (P - 2\cdot FN_{\mathcal{S}_+}) \cdot FP_{\mathcal{S}_+} \\
& \ge N \cdot FN_{\mathcal{S}_+} + (P - 2\cdot FN_{\mathcal{S}_+}) \cdot FP_{\mathcal{S}} \\
& = (N-2\cdot FP_{\mathcal{S}}) \cdot FN_{\mathcal{S}_+}  + P \cdot FP_{\mathcal{S}} \\
& \ge P \cdot FP_{\mathcal{S}}
\end{split}
\end{align*}
In the case $FN_\mathcal{S} \ge P/2 \text{ and } FP_\mathcal{S} > N/2$,  we have $N - 2\cdot FP_{\mathcal{S}_+} \le N - 2\cdot FP_{\mathcal{S}} < 0$ and $P - 2\cdot FN_{\mathcal{S}} < 0$,  therefore,
\begin{align*}
\begin{split} \nu(\mathcal{S}_+) & = P \cdot FP_{\mathcal{S}_+} + (N-2\cdot FP_{\mathcal{S}_+}) \cdot FN_{\mathcal{S}_+} \\
& \ge P \cdot FP_{\mathcal{S}_+} + (N-2\cdot FP_{\mathcal{S}_+}) \cdot FN_{\mathcal{S}} \\
& = (P - 2\cdot FN_{\mathcal{S}}) \cdot FP_{\mathcal{S}_+} + N \cdot FN_{\mathcal{S}} \\
& \ge (P - 2\cdot FN_{\mathcal{S}_+}) \cdot N + N \cdot FN_{\mathcal{S}} \\
& = (P - FN_{\mathcal{S}}) \cdot N
\end{split}
\end{align*}
In the last case where $FN_\mathcal{S} \ge P/2 \text{ and } FP_\mathcal{S} \le N/2$,  there is insufficient information to derive a nontrivial lower bound for $\nu(\mathcal{S}_+)$ while $\nu(\mathcal{S}_+) \ge 0$ always holds. 
\end{proof}

Let us denote the lower bound for $\nu(\mathcal{S}_+)$ by $\tau(\mathcal{S})$ to emphasize its sole dependence on the current solution $\mathcal{S}$,  and assign its value according to Proposition \ref{prop1}. 
In the search for the optimal solution,  whenever $\tau(\mathcal{S})$ is greater than the best objective value found so far,  any solution that would result from enlarging $\mathcal{S}$ can be eliminated.  
The algorithm is outlined as follows.  We start with evaluating all single-variable split rules,  i.e., $\mathcal{S}_k = \{k\}$, for $k \in \mathcal{K}$,  by calculating their $\nu(\mathcal{S}_k)$ and $\tau(\mathcal{S}_k)$.  These are the ``root nodes'' in the search tree \footnote{The phrase ``search tree'' is an analogy - in implementation a queue is used for storing yet-to-explore solutions. }. Let $\nu^\text{best}$ be the smallest objective values encountered so far.  We can eliminate those nodes $k$ having $\tau(\mathcal{S}_k) > \nu^\text{best}$,  and we can terminate a node (i.e., making it a leaf in the search tree) when $\tau(\mathcal{S}_k) = \nu(\mathcal{S}_k)$.  For each remaining root node $k$,  we evaluate all two-variable split rules branched from it in such a way that the added variable's index is greater than $k$ (to avoid redundant evaluations),  and update the $\nu^\text{best}$ and eliminate unpromising branches on the fly.  For instance,  if $\tau(\{1,2\}) > \nu^\text{best}$, then the candidate solution $\{1,2\} \cup K$ for each $K \subset 2^{\{3, \ldots, m\}}$ can be eliminated.  The search proceeds until all possibilities are examined,  at which point the $\nu^\text{best}$ is the optimal solution to OPT-G.  This algorithm is implemented in the {\ttfamily bslearn} function in the R source code under the ``enum'' solver option,  and we call it the ENUM algorithm in the rest of this paper. 
\begin{figure}
\centering
\begin{tikzpicture}[sibling distance=5em,
  every node/.style = {align=center, font = \normalsize}]]
  \node (1) {1}
    child { node {2}
        child { node {3}
            child { node {4}}}
        child { node {4}}}
    child { node {3}
            child { node {4}}}
    child { node {4}};
  \begin{scope}[xshift=4.5cm]
      \node {2}
        child { node {3}
            child { node {4}}}
        child { node {4}};
  \end{scope}
  \begin{scope}[xshift=7cm]
      \node {3}
        child { node {4}};
  \end{scope}
  \begin{scope}[xshift=9cm]
    \node (4) {4};
  \end{scope}
\end{tikzpicture}
\caption{Enumeration of all possible solutions in a solution space of four candidate rules indexed by $\{1, 2, 3, 4\}$.  Each node,  along with the top-down path leading to it, represents a unique subset of the solution space.  Candidate solutions are evaluated in the breadth-first order.  Proposition \ref{prop1} enables opportunities to prune the search tree branches. } \label{fig:ENUM_search_demo} 
\end{figure}
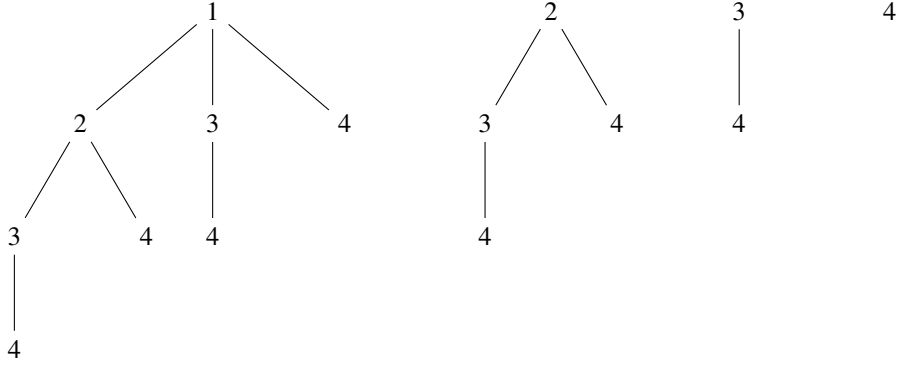

The ENUM algorithm, if carried out to completion,  can guarantee to return a solution having the smallest objective value.  
In essence,  the sequential evaluation of candidate solutions reduces the search space from $\mathcal{K} \times \mathcal{P} \times \mathcal{N}$ to $2^\mathcal{K}$.  Since the cost of evaluating a candidate solution ramps up slowly with the size of $\mathcal{I}$ (i.e., $\mathcal{P} \cup \mathcal{N}$) thanks to efficient vector operations,  the ENUM can find the optimal solution faster than the typical branch-and-bound method of MIP solvers,  when $\mathcal{I}$ is large and $\mathcal{K}$ is relatively small.  

Compared to using an MIP solver,  ENUM boasts the following advantages: 
(1) it scales betters in terms of memory cost because no branch-and-bound tree is maintained; 
(2) the search process can be parallelized (though it is a bit tricky to implement in R since R is intrinsically single-threaded); (3) the search is breadth first,  meaning that split rules having fewer clauses are evaluated before rules having more clauses get evaluated.  Therefore,  if any time the search is terminated prematurely (e.g., due to time limit),  a simplest possible best found rule can still be returned; in addition,  if the optimal solution is not unique,  one having the fewest clauses will be returned.  


\subsection{Complexity regulation constraints}
When OPT-G is used to split the nodes,  the classification tree can be grown until every node is pure (i.e., containing observations of the same class).  Such a tree,  called a maximal tree,  does not generalize well on new data.  Additional constraints can be added to OPT-G to regulate the complexity of the tree to curb overfitting.  In the {\ttfamily bsnsing} package, we leverage two types of constraints as follows. 
\begin{align}
& \sum_{k \in \mathcal{K}} w_k \le \text{MaxRules} \label{eq:maxrules} \\
& \sum_{i \in \mathcal{I}} z_i \ge \text{MinNodeSize} \label{eq:nodesize1}\\
& \sum_{i \in \mathcal{I}} (1 - z_i) \ge \text{MinNodeSize} \label{eq:nodesize2}
\end{align}

The right-hand sides of these inequalities are control parameters.  Constraint (\ref{eq:maxrules}) limits the number of questions to enter a split rule.  This constraint is quite necessary in practice especially when the tree is to be used for prediction,  because devoid of this constraint each split optimization step would amount to a maximal overfitting of the data available in the present node - finding a single composite split rule to maximally reduce the Gini.  In addition,  this constraint also directly reduces the solution space thus its presence expedites the ENUM search.  Specifically,  the number of objective function evaluations in the worst case (assuming no pruning and early termination opportunities exist in the search process) would reduce to $\sum_{i = 1}^\text{MaxRules}  \binom{m}{i}$ as compared to $2^m$ when the constraint was not present.  For instance,  for a pool of $m = 98$ candidate questions,  constraint (\ref{eq:maxrules}) with $\text{MaxRules} =3$ would reduce the worst-case objective function evaluations to $\binom{98}{1} + \binom{98}{2} + \binom{98}{3} = 156947$ from the theoretical worst worst case of $2^{98} = 3.17 \times 10^{29}$.  A demonstration is presented in Section \ref{sec:evaluation}.  

Constraints (\ref{eq:nodesize1}) and (\ref{eq:nodesize2}) require that each child node from a split must contain at least $\text{MinNodeSize}$ (a positive number) observations.  They are effective at limiting the tree depth as well as generating well-populated leaf nodes.  In the {\ttfamily bsnsing} package, the default value for \text{MaxRules} is 2,  and MinNodeSize is by default set equal to the square root of $n$,  the number of training examples. 

In certain use cases, it is customary to expect the child nodes of a split to bear different majority classes, i.e.,  to require $TP/(TP+FP) \ge 0.5$ and $FN/(FN+TN) \le 0.5$.  
This requirement can be translated to the following linear constraint for the OPT-G model. 
\begin{align}
& \sum_{i \in \mathcal{P}} z_i - \sum_{j \in \mathcal{N}} z_j \ge \max\{0, 2\cdot P - n\} \label{eq:gender} 
\end{align}
To include this constraint in the ENUM algorithm,  we can simply discard any candidate solution that violates it in the search process.  We comment that this constraint is mainly for the convenience of tree interpretation, and would lead to an over-regulation (i.e., creating unnecessary bias) to the tree model.  Specifically,  its inclusion in the model tends to produce simple and shallow trees that lack fidelity in discriminating new cases.  Therefore, we choose not to enable it by default in the {\ttfamily bsnsing} package.  To enable the constraint,  the user can explicitly set the parameter {\ttfamily no.same.gender.children} to True. 

Other hyperparameters used in the {\ttfamily bsnsing} function include (1) the bin size ({\ttfamily bin.size}), which specifies the minimum number of observations that must satisfy a candidate binary question for the question to enter the pool; (2) the stop probability ({\ttfamily stop.prob}),  a node purity threshold in terms of the proportion of the majority class in the node which, if exceeded,  the node will not be further split; (3) the maximum number of segments the range of a numeric variable ({\ttfamily nseg.numeric}) is divided into by inequalities of the same direction.  These parameters do not directly affect the split rule optimization, but they affect the overall efficiency of the tree building process and the final performance of the tree.  Apart from generating candidate binary questions internally, the {\ttfamily bsnsing} function is also able to import split questions generated by other decision tree packages (currently including C50,  tree,  party and rpart packages) into its own pool for optimal selection.  This option is enabled by the parameter {\ttfamily import.external}.  For a complete list of control parameters,  users could consult the help document by typing ``{\ttfamily ?bscontrol}'' in R.

\section{Evaluation} \label{sec:evaluation}
The models and algorithms developed in this paper are implemented in the open-source R package named {\ttfamily bsnsing}.  The source code is hosted at \href{github.com/profyliu/bsnsing}{github.com/profyliu/bsnsing}.   

In this section, we demonstrate the computational efficiency of solving OPT-G using different algorithms,  compare the performance of {\ttfamily bsnsing} against several other decision tree packages available in R, and showcase the basic usage of the {\ttfamily bsnsing} library via some code samples.  
The experiments were performed in RStudio (R version 3.6.2) on a MacBook Pro with Intel Core i9 (8 cores) processor and 16GB RAM.

\subsection{Computational efficiency of solving OPT-G}
We perform two sets of experiments to demonstrate the superiority of the ENUM algorithm over the Gurobi solver for solving OPT-G. In the first set of experiments, we generate OPT-G instances of different sizes based on the seismic data set from the UCI machine learning repository \citep{Dua:2019} and show how each method scales as data size increases; in the second set of experiments, we contrast the solutions by the two methods on more classification data sets adopted from the literature, to further solidify our conclusion.  

\subsubsection{Experiments on the seismic data set.}
The seismic data set has 1690 observations and 18 input variables, among which 4 are categorical variables and 14 are numeric variables.  The output is a categorical variable with two levels,  so it is a binary classification problem.  We first binarize all inputs to create a binary matrix $B$ (using the {\ttfamily binarize} function in the {\ttfamily bsnsing} package),  essentially substituting one or multiple binary features for each of the original input variables.  The matrix $B$ has $1690$ rows and $100$ columns. 
We experiment two scenario factors that are critical to the computational efficiency,  the number of training cases $n$ and  the maximum number of rules,  i.e., $\text{MaxRules}$,  allowed in the solution.  We randomly sample $n$ rows from $B$ and join them with the corresponding response variables to form sets of training data with varying sizes.   For each combination of {\ttfamily max.rules} in $\{1,2,3,4\}$ and $n$ in $\{200, 400, \ldots, 1600\}$,  we solve the OPT-G model using both the Gurobi solver and the ENUM algorithm. For example, suppose the selected rows of matrix $B$ is stored in matrix {\ttfamily bx} and the binary target vector is {\ttfamily y},  the following R code was used to learn the optimal split rule with MaxRules = 3 using the Gurobi solver: 

\noindent {\ttfamily
res <- bslearn(bx, y, bscontrol(opt.solver = `gurobi', solver.timelimit = 7200, max.rules = 3, node.size = 1))
}

\noindent Note that the {\ttfamily node.size} option was set to 1 to relax the constraints (\ref{eq:nodesize1}) and (\ref{eq:nodesize2}) that were irrelevant to this experiment. For the ENUM algorithm, the {\ttfamily opt.solver} option was set to {\ttfamily `enum\_c'}\footnote{Setting {\ttfamily opt.solver=`enum\_c'} invokes the compiled code (written in C) implementing the ENUM algorithm, as opposed to using the plain R implementation which can be invoked by setting {\ttfamily opt.solver = `enum'}. } . 

The results are summarized in Table \ref{tb:ENUM_neval}. As expected, in each case both methods were able to find the same optimal objective value, reported in the column Objval. For the same instance, the Objval was non-increasing with the increase in max.rules, which also matches our expectation. 

The next three columns of Table \ref{tb:ENUM_neval} list the computing time in seconds. The Gurobi solver automatically exploited multiple CPU cores available in the computer (which had 8 physical cores),  so we report both the CPU time (actual computing resource usage) in column Gurobi and the elapsed time (wall time as experienced by the user) in column Gurobi.E. The ENUM algorithm used only one CPU core, thus only the elapsed time is reported. Clearly, ENUM is the incontestable winner, running at least two orders of magnitude faster than Gurobi in elapsed time. 


For the ENUM algorithm, the number of objective function evaluations in the worst case is $\sum_{i=1}^\text{max.rules} \binom{100}{i}$.  However,  the actual number of evaluations was significantly fewer than the theoretical worst case.  The actual numbers of objective function evaluations,  as well as their percentage of the theoretical worst case,  are listed in columns Obj.Evals and Pct of All Feasible, respectively.  For example,  when $n = 200$ and {\ttfamily max.rule} = 1,  ENUM evaluated 100 candidate solutions to find the optimal one,  and this number is 100\% of all candidate solutions in the search space.  We can see that the percentage dropped significantly (hence substantial savings in computing accrued) with the increase in {\ttfamily max.rules}.  The savings are attributed to the bounding and early termination strategy of Proposition \ref{prop1}.  Similar experiments on other data sets and other solvers (i.e.,  CPLEX and lpSolve) would reveal the same insight, so we forgo repeated experiments on those scenarios.



\subsubsection{Experiments on selected DL8.5 datasets.}
In validating the DL8.5 algorithm and demonstrating its superiority over the BinOCT algorithm \citep{VerwerZhang2019}, \cite{Aglin_Nijssen_Schaus_2020} employed 24 binary classification data sets, all consisting of pure binary features. 
These data sets can directly form instances of OPT-G without the need for feature binarization. 
We selected 13 from the 24 data sets by the criterion $n \le 1000$ and $p \le 200$, and compared Gurobi and ENUM solutions to OPT-G for the root node split rule identification on these data sets. The reason for the above selection criterion is that Gurobi cannot solve the larger cases in a tolerable amount of time, i.e., two hours.  The results are presented in Table \ref{tb:gurobi_enum_dl85}.  The problem size is noted by $n$ and $m$ (for these cases, $p = m$). Sharp contrasts in solution time - up to a ratio of $10^6$ - between the two methods persisted throughout all cases. In addition, for the australian-credit and tic-tac-toe cases, Gurobi failed to terminate within the 2-hour time limit, ending with a suboptimal solution in the latter case. In comparison, ENUM was able to find the optimal solutions consistently in less than 0.1 second. In cases where Gurobi succeeded, both methods returned the same optimal rule. These experiments serve to solidify our conclusion that ENUM should be the solver of choice when the {\ttfamily bsnsing} package is employed in practice.

\begin{table}[!htb]
\centering
\small
\caption{Computational costs and scalability of ENUM. } \label{tb:ENUM_neval}
\begin{tabular}{c c r r r r r r}
\hline
Obs.  & max.rules & Objval & Gurobi & Gurobi.E & ENUM & Obj.  Evals & Pct of All Feasible \\
\hline 
\multirow{4}{*}{200} & 1 & 512 & 5.6 & 2.8 & 0 & 100 & 100\% \\
 & 2 & 492 & 14.5 & 2.9 & 0 & 2959 & 58.6\% \\
 & 3 & 420 & 20.4 & 3.8 & 0.1 & 36998 & 22.2\% \\
 & 4 & 378 & 20 & 3.7 & 0.5 & 249951 & 6.1\% \\
\hline 
\multirow{4}{*}{400} & 1 & 2324 & 14.5 & 6.8 & 0 & 100 & 100\% \\
 & 2 & 2046 & 97.7 & 14.9 & 0 & 3163 & 62.6\% \\
 & 3 & 1974 & 188.7 & 27.6 & 0.2 & 46270 & 27.7\% \\
 & 4 & 1974 & 93.3 & 14.8 & 2 & 406135 & 9.9\% \\
\hline 
\multirow{4}{*}{600} & 1 & 5044 & 48.5 & 19.7 & 0 & 100 & 100\% \\
 & 2 & 4710 & 261.1 & 38.6 & 0 & 3332 & 66\% \\
 & 3 & 4610 & 421.3 & 59.2 & 0.3 & 49059 & 29.4\% \\
 & 4 & 4610 & 359.4 & 50.5 & 3.8 & 445965 & 10.9\% \\
\hline 
\multirow{4}{*}{800} & 1 & 8494 & 117.6 & 43.9 & 0 & 100 & 100\% \\
 & 2 & 8153 & 796.7 & 109.2 & 0.1 & 3308 & 65.5\% \\
 & 3 & 7933 & 1447.3 & 191.8 & 0.6 & 51389 & 30.8\% \\
 & 4 & 7802 & 787.1 & 108.2 & 5.4 & 463565 & 11.3\% \\
\hline 
\multirow{4}{*}{1000} & 1 & 12410 & 58.5 & 34.2 & 0 & 100 & 100\% \\
 & 2 & 11946 & 1359.8 & 188.2 & 0.1 & 3241 & 64.2\% \\
 & 3 & 11652 & 1727.2 & 233.7 & 0.7 & 50300 & 30.2\% \\
 & 4 & 11382 & 2498.2 & 327 & 6.6 & 442363 & 10.8\% \\
\hline 
\multirow{4}{*}{1200} & 1 & 17027 & 842.7 & 433.8 & 0 & 100 & 100\% \\
 & 2 & 15840 & 2370.9 & 316.3 & 0.1 & 3282 & 65\% \\
 & 3 & 15457 & 3094.5 & 409.8 & 0.9 & 54261 & 32.5\% \\
 & 4 & 15316 & 4143.4 & 543.3 & 9 & 535880 & 13.1\% \\
\hline 
\multirow{4}{*}{1400} & 1 & 25354 & 625.7 & 151.4 & 0 & 100 & 100\% \\
 & 2 & 23302 & 3299.4 & 442.6 & 0.1 & 3373 & 66.8\% \\
 & 3 & 22794 & 4856 & 639.5 & 1 & 53873 & 32.3\% \\
 & 4 & 22578 & 6905.6 & 895.9 & 10.1 & 540219 & 13.2\% \\
\hline 
\multirow{4}{*}{1600} & 1 & 34664 & 558.9 & 293.4 & 0.1 & 100 & 100\% \\
 & 2 & 32460 & 19969.1 & 2572.5 & 0.1 & 3282 & 65\% \\
 & 3 & 31887 & 21834.1 & 2804.5 & 1.2 & 54642 & 32.8\% \\
 & 4 & 31737 & 33050.8 & 4212.2 & 11.2 & 551884 & 13.5\% \\
 \hline
\end{tabular}
\end{table}

\begin{table}
\centering
\small
\caption{Comparison between Gurobi and ENUM for root node split on DL85 datasets with {\ttfamily max.rules = 2}.} \label{tb:gurobi_enum_dl85}
\begin{tabular}{l r r@{\hskip 0.3in} r r r r c r}
\hline
 & $n$ & $m$ & Method & CPU & Elapsed & Objval & Rule & nEval \\
\hline
\multirow{2}{*}{anneal} & \multirow{2}{*}{812} & \multirow{2}{*}{93} & ENUM & 0.045 & 0.083 & 47750 & V60 $|$ V67 & 2204 \\
 & & & Gurobi & 9005.889 & 1181.632 & 47750 & V60 $|$ V67 & 0 \\
\hline 
\multirow{2}{*}{audiology} & \multirow{2}{*}{216} & \multirow{2}{*}{148} & ENUM & 0.023 & 0.027 & 885 & V2 $|$ V32 & 5768 \\
 & & & Gurobi & 63.565 & 10.994 & 885 & V2 $|$ V32 & 0 \\
\hline 
\multirow{2}{*}{australian-credit} & \multirow{2}{*}{653} & \multirow{2}{*}{125} & ENUM & 0.057 & 0.059 & 24456 & V34 $|$ V76 & 4493 \\
 & & & Gurobi & 56229.196 & 7204.716 & 24456 & V38 $|$ V76 & 0 \\
\hline 
\multirow{2}{*}{breast-wisconsin} & \multirow{2}{*}{683} & \multirow{2}{*}{120} & ENUM & 0.051 & 0.053 & 9039 & V21 $|$ V77 & 3427 \\
 & & & Gurobi & 5276.419 & 691.689 & 9039 & V21 $|$ V77 & 0 \\
\hline 
\multirow{2}{*}{diabetes} & \multirow{2}{*}{768} & \multirow{2}{*}{112} & ENUM & 0.063 & 0.063 & 53312 & V20 $|$ V76 & 3180 \\
 & & & Gurobi & 55471.42 & 7205.7 & 53312 & V20 $|$ V76 & 0 \\
\hline 
\multirow{2}{*}{heart-cleveland} & \multirow{2}{*}{296} & \multirow{2}{*}{95} & ENUM & 0.023 & 0.025 & 7460 & V91 $|$ V96 & 2483 \\
 & & & Gurobi & 3190.064 & 418.003 & 7460 & V91 $|$ V96 & 0 \\
\hline 
\multirow{2}{*}{hepatitis} & \multirow{2}{*}{137} & \multirow{2}{*}{68} & ENUM & 0.006 & 0.007 & 876 & V37 $|$ V51 & 1183 \\
 & & & Gurobi & 109.495 & 15.128 & 876 & V37 $|$ V51 & 0 \\
\hline 
\multirow{2}{*}{lymph} & \multirow{2}{*}{148} & \multirow{2}{*}{68} & ENUM & 0.008 & 0.008 & 1655 & V39 $|$ V56 & 1562 \\
 & & & Gurobi & 214.038 & 28.534 & 1655 & V39 $|$ V56 & 0 \\
\hline 
\multirow{2}{*}{primary-tumor} & \multirow{2}{*}{336} & \multirow{2}{*}{31} & ENUM & 0.006 & 0.006 & 7728 & V21 $|$ V29 & 284 \\
 & & & Gurobi & 530.928 & 70.518 & 7728 & V21 $|$ V29 & 0 \\
\hline 
\multirow{2}{*}{soybean} & \multirow{2}{*}{630} & \multirow{2}{*}{50} & ENUM & 0.016 & 0.017 & 19964 & V24 $|$ V36 & 997 \\
 & & & Gurobi & 2304.774 & 306.424 & 19964 & V24 $|$ V36 & 0 \\
\hline 
\multirow{2}{*}{tic-tac-toe} & \multirow{2}{*}{958} & \multirow{2}{*}{27} & ENUM & 0.012 & 0.012 & 95336 & V15 & 378 \\
 & & & Gurobi & 14160.407 & 7202.697 & 105464 & V28 & 0 \\
\hline 
\multirow{2}{*}{vote} & \multirow{2}{*}{435} & \multirow{2}{*}{48} & ENUM & 0.011 & 0.013 & 3547 & V12 & 631 \\
 & & & Gurobi & 1027.852 & 346.244 & 3547 & V12 & 0 \\
\hline 
\multirow{2}{*}{zoo-1} & \multirow{2}{*}{101} & \multirow{2}{*}{36} & ENUM & 0.002 & 0.002 & 0 & V8 & 36 \\
 & & & Gurobi & 0.277 & 0.283 & 0 & V8 & 0 \\
\hline 
\end{tabular}
\end{table}

\subsection{Comparison with the DL8.5, OSDT and GOSDT algorithms} \label{sec:comp_dl85_osdt}
In this section, we compare {\ttfamily bsnsing} against three recently developed optimal decision tree (ODT) methods, namely, DL8.5 \citep{Aglin_Nijssen_Schaus_2020}, OSDT (Optimal Sparse Decision Trees) \citep{DBLP:journals/corr/abs-1904-12847} and GOSDT (Generalized and scalable optimal sparse decision trees) \citep{lin2020generalized}. The software programs were obtained through the Github links provided in the respective papers.  
Given that the model assumptions, formulation and parameters are all different for the different tools under comparison, we do not aim to provide a comprehensive evaluation, by, for instance, performing problem-specific parameter tuning and model interpretation, or making inferences about which method is most suitable for what kind of data and applications. Instead, we will exhibit our computational experience from a user's perspective, and let it convey the unique position of {\ttfamily bsnsing} among other recent ODT tools. 

We experiment on the 24 data sets (i.e., binary classification problems with binary features) used in the DL8.5 paper \citep{Aglin_Nijssen_Schaus_2020}. 
	The experiments were performed as follows. The comparison experiment on each data set was repeated 20 times. In each repetition, using an arbitrarily sequenced seed for the random number generator (RNG), we randomly split the data set into two parts, 70\% for training and 30\% for testing. For each method, we recorded the classification accuracy on the test set, as well as the time taken (in seconds) to train the respective models. For {\ttfamily bsnsing}, we adopted the default parameters, specifically, {\ttfamily opt.solver = `enum\_c', max.rules = 2}, and set the {\ttfamily node.size} to $n^{1/4}$, where $n$ is the training sample size; for DL8.5, we adopted the default parameters as recommended in the package's companion paper and in its online sample code, that is, ``max\_depth=3, min\_sup=5''; for OSDT, we adopted the default values for all optional parameters and set the required parameters to ``lamb=0.005, prior\_metric=`curiosity''', by consulting the sample code provided in the author's Github page; and for GOSDT, we set the required parameter ``regulerization'' to 0.001\footnote{The regularization parameter in GOSDT is the multiplier on the number-of-leaves term in the two-term objective function to be minimized, whereas the other term is the training accuracy (with multiplier 1).}, a value found suitable in several experiments in the extended manuscript of the method, see the appendices of \cite{lin2020generalized_long}. 
In addition, for OSDT and GOSDT we set the time limit to 5 minutes (via setting the parameter ``timelimit=300'' in OSDT and ``time\_limit=300'' in GOSDT), because we observed (which was also acknowledged by OSDT's developer) that the memory usage of the OSDT and GOSDT training process increases linearly and quickly with execution time.

The comparison results are listed in Table \ref{tb:dl85_bsnsing_compare}. The columns $n$, $p$ and MR characterize the data sets, where MR represents the rate of the minority class. The best accuracy (averaged over 20 runs) for each data set is highlighted in boldface. There is not a clear winner (or loser) in terms of prediction accuracy among the four methods - overall, all methods seem comparably capable. However, the differences in training time is significant. First, we can observe that OSDT and GOSDT are less competitive in training time. Another factor that discounts OSDT algorithm's actual performance is that, the OSDT program internally uses scikit-learn's decision tree classifier (which implements the CART algorithm) as its baseline predictor, and will return the CART model if the OSDT algorithm cannot do better. In most cases, the OSDT program indeed returned the CART model for prediction (the CART column in Table \ref{tb:dl85_bsnsing_compare} notes the proportion of runs in which the CART model was returned by the OSDT method), indicating that the OSDT algorithm did not outperform the CART baseline in these cases.

\begin{table}
\centering
\footnotesize
\caption{Computational comparison of four ODT packages.} \label{tb:dl85_bsnsing_compare}
\begin{tabular}{l r r r@{\hskip 0.3in} r r r@{\hskip 0.3in} r r@{\hskip 0.3in} r r r@{\hskip 0.3in} r r}
\hline
\multirow{2}{*}{} & \multirow{2}{*}{$n$} & \multirow{2}{*}{$p$} & \multirow{2}{*}{MR} & \multicolumn{3}{c}{bsnsing} & \multicolumn{2}{c}{DL8.5} & \multicolumn{3}{c}{OSDT} & \multicolumn{2}{c}{GOSDT} \\
\cline{5-14}
 & & & & Accu & CPU & Dp & Accu & CPU & Accu & CPU & CART & Accu & CPU \\
\hline
anneal & 812 & 93 & 0.230 & 0.844 & 2.6 &  7  & \textbf{0.847} & 1.1 & 0.839 & 305.0 & 0.8 & 0.846 & 304.4 \\  
audiology & 216 & 148 & 0.264 & 0.925 & 0.4 &  3  & 0.925 & 0.2 & \textbf{0.938} & 297.9 & 0.3 & 0.934 & 306.0 \\  
australian-credit & 653 & 125 & 0.453 & 0.823 & 2.3 &  6  & 0.853 & 4.6 & \textbf{0.858} & 304.1 & 0.7 & 0.857 & 311.4 \\  
breast-wisconsin & 683 & 120 & 0.350 & 0.953 & 0.8 &  4  & \textbf{0.955} & 3.1 & 0.951 & 304.6 & 0.3 & 0.943 & 307.8 \\  
diabetes & 768 & 112 & 0.349 & 0.717 & 3.8 &  7  & 0.737 & 5.4 & \textbf{0.750} & 305.7 & 0.8 & 0.749 & 314.5 \\  
german-credit & 1000 & 112 & 0.300 & 0.676 & 5.0 &  7  & \textbf{0.728} & 4.2 & 0.723 & 305.2 & 0.6 & 0.708 & 319.9 \\  
heart-cleveland & 296 & 95 & 0.459 & 0.752 & 1.2 &  5  & \textbf{0.766} & 1.9 & 0.763 & 303.7 & 0.9 & 0.746 & 306.9 \\  
hepatitis & 137 & 68 & 0.190 & 0.790 & 0.4 &  4  & 0.802 & 0.5 & 0.775 & 306.0 & 1.0 & \textbf{0.829} & 303.1 \\  
hypothyroid & 3247 & 88 & 0.085 & 0.975 & 2.3 &  6  & 0.979 & 2.2 & \textbf{0.979} & 304.7 & 0.4 & \textbf{0.979} & 307.0 \\  
ionosphere & 351 & 445 & 0.359 & 0.881 & 4.2 &  4  & 0.869 & 235.0 & 0.888 & 302.8 & 0.7 & \textbf{0.912} & 519.6 \\  
kr-vs-kp & 3196 & 73 & 0.478 & \textbf{0.984} & 2.0 &  7  & 0.936 & 1.3 & 0.969 & 303.9 & 1.0 & 0.853 & 307.9 \\  
letter & 20000 & 224 & 0.041 & \textbf{0.990} & 39.5 &  8  & 0.981 & 304.4 & 0.959 & 3033.1 & 1.0 & 0.959 & 631.6 \\  
lymph & 148 & 68 & 0.453 & \textbf{0.810} & 0.3 &  4  & 0.801 & 0.2 & 0.798 & 305.2 & 0.9 & 0.761 & 303.3 \\  
mushroom & 8124 & 119 & 0.482 & \textbf{0.999} & 2.7 &  4  & 0.999 & 4.4 & 0.994 & 305.1 & 0.0 & 0.943 & 314.4 \\  
pendigits & 7494 & 216 & 0.104 & \textbf{0.994} & 10.9 &  5  & 0.991 & 94.4 & 0.989 & 303.4 & 0.0 & 0.970 & 381.2 \\  
primary-tumor & 336 & 31 & 0.244 & 0.778 & 0.7 &  6  & \textbf{0.843} & 0.1 & 0.822 & 304.9 & 0.2 & 0.751 & 169.6 \\  
segment & 2310 & 235 & 0.143 & \textbf{0.996} & 1.6 &  3  & 0.996 & 12.8 & 0.995 & 4.6 & 0.9 & 0.996 & 311.0 \\  
soybean & 630 & 50 & 0.146 & 0.935 & 0.7 &  6  & \textbf{0.941} & 0.2 & 0.938 & 306.2 & 0.3 & 0.862 & 304.6 \\  
splice-1 & 3190 & 287 & 0.481 & \textbf{0.946} & 18.2 &  7  & 0.926 & 51.6 & 0.946 & 303.0 & 1.0 & 0.835 & 308.4 \\  
tic-tac-toe & 958 & 27 & 0.347 & 0.875 & 1.0 &  7  & 0.733 & 0.1 & \textbf{0.901} & 304.1 & 1.0 & 0.757 & 305.3 \\  
vehicle & 846 & 252 & 0.258 & 0.952 & 3.1 &  5  & \textbf{0.956} & 33.6 & 0.946 & 303.5 & 0.8 & 0.879 & 354.0 \\  
vote & 435 & 48 & 0.386 & 0.940 & 0.3 &  4  & 0.943 & 0.2 & \textbf{0.956} & 304.2 & 0.1 & 0.950 & 304.7 \\  
yeast & 1484 & 89 & 0.312 & 0.698 & 7.1 &  8  & 0.692 & 3.1 & 0.701 & 305.8 & 0.3 & \textbf{0.701} & 308.9 \\  
zoo-1 & 101 & 36 & 0.406 & 0.995 & 0.0 &  1  & 0.995 & 0.0 & \textbf{1.000} & 0.2 & 1.0 & 0.995 & 0.0 \\  
\hline
\end{tabular}
\end{table}

\begin{table}
\centering
\footnotesize
\caption{Performance comparison under depth constraints. } \label{tb:dl85_compare_depth1}
\begin{tabular}{l r r r r@{\hskip 0.3in} r r r r@{\hskip 0.3in} r r r r}
\hline
&  \multicolumn{4}{c}{depth = 1} &  \multicolumn{4}{c}{depth = 2} &  \multicolumn{4}{c}{depth = 3} \\
\cline{2-13}
 & bs(2) & bs(1) & DL8.5 & OSDT & bs(2) & bs(1) & DL8.5 & OSDT & bs(2) & bs(1)  & DL8.5 & OSDT \\
\hline
anneal & 0.778 & 0.778 & \textbf{0.818} & 0.818 & 0.786 & 0.778 & \textbf{0.826} & 0.824 & 0.806 & 0.777 & \textbf{0.847} & 0.836 \\  
audiology & \textbf{0.929} & 0.856 & 0.856 & 0.856 & 0.925 & 0.924 & \textbf{0.934} & \textbf{0.934} & 0.925 & 0.923 & 0.925 & \textbf{0.934} \\  
australian-credit & 0.854 & \textbf{0.866} & \textbf{0.866} & \textbf{0.866} & 0.854 & \textbf{0.866} & 0.851 & 0.862 & 0.849 & 0.852 & \textbf{0.853} & 0.849 \\  
breast-wisconsin & \textbf{0.933} & 0.919 & 0.923 & 0.918 & 0.940 & 0.921 & 0.960 & \textbf{0.960} & \textbf{0.956} & 0.951 & 0.955 & 0.955 \\  
diabetes & 0.716 & 0.725 & \textbf{0.751} & \textbf{0.751} & 0.754 & 0.734 & 0.754 & \textbf{0.762} & 0.747 & \textbf{0.760} & 0.737 & 0.755 \\  
german-credit & 0.696 & 0.699 & \textbf{0.703} & 0.703 & 0.709 & 0.700 & \textbf{0.717} & 0.717 & 0.711 & 0.703 & \textbf{0.728} & 0.725 \\  
heart-cleveland & 0.736 & 0.737 & \textbf{0.737} & 0.737 & \textbf{0.753} & 0.735 & 0.735 & 0.729 & 0.774 & \textbf{0.811} & 0.766 & 0.798 \\  
hepatitis & 0.798 & 0.802 & \textbf{0.844} & 0.840 & 0.812 & 0.814 & \textbf{0.831} & 0.817 & 0.787 & 0.794 & \textbf{0.802} & 0.800 \\  
hypothyroid & \textbf{0.963} & \textbf{0.963} & \textbf{0.963} & \textbf{0.963} & 0.973 & 0.963 & \textbf{0.979} & \textbf{0.979} & 0.977 & 0.969 & 0.979 & \textbf{0.979} \\  
ionosphere & \textbf{0.912} & 0.758 & 0.820 & 0.820 & 0.894 & 0.837 & 0.901 & \textbf{0.904} & 0.875 & 0.851 & 0.869 & \textbf{0.891} \\  
kr-vs-kp & \textbf{0.772} & 0.684 & 0.678 & 0.678 & \textbf{0.934} & 0.868 & 0.868 & 0.868 & \textbf{0.949} & 0.923 & 0.936 & 0.903 \\  
letter & \textbf{0.959} & \textbf{0.959} & \textbf{0.959} & \textbf{0.959} & 0.959 & 0.959 & \textbf{0.969} & 0.968 & \textbf{0.982} & 0.958 & 0.981 & 0.959 \\  
lymph & 0.743 & \textbf{0.772} & 0.752 & \textbf{0.772} & \textbf{0.780} & 0.777 & 0.770 & 0.766 & \textbf{0.811} & 0.803 & 0.801 & 0.784 \\  
mushroom & \textbf{0.951} & 0.887 & 0.887 & 0.887 & \textbf{0.983} & 0.915 & 0.969 & 0.969 & \textbf{0.999} & 0.969 & 0.999 & 0.994 \\  
pendigits & \textbf{0.967} & 0.895 & 0.931 & 0.931 & \textbf{0.988} & 0.978 & 0.978 & 0.978 & 0.991 & 0.987 & \textbf{0.991} & 0.987 \\  
primary-tumor & 0.765 & \textbf{0.775} & 0.765 & 0.769 & \textbf{0.806} & \textbf{0.806} & 0.801 & 0.803 & 0.803 & 0.797 & \textbf{0.843} & 0.836 \\  
segment & \textbf{0.991} & 0.926 & 0.981 & 0.981 & \textbf{0.996} & 0.995 & 0.995 & 0.990 & \textbf{0.996} & 0.995 & 0.996 & 0.995 \\  
soybean & \textbf{0.862} & \textbf{0.862} & \textbf{0.862} & \textbf{0.862} & 0.854 & 0.862 & 0.906 & \textbf{0.913} & 0.913 & 0.890 & \textbf{0.941} & 0.920 \\  
splice-1 & \textbf{0.824} & 0.816 & 0.816 & 0.816 & \textbf{0.875} & 0.835 & 0.827 & 0.831 & \textbf{0.943} & 0.909 & 0.926 & 0.908 \\  
tic-tac-toe & 0.665 & 0.669 & \textbf{0.703} & \textbf{0.703} & \textbf{0.693} & 0.683 & 0.673 & 0.678 & 0.744 & \textbf{0.745} & 0.733 & 0.735 \\  
vehicle & \textbf{0.865} & 0.736 & 0.755 & 0.755 & \textbf{0.913} & 0.895 & 0.904 & 0.904 & 0.944 & 0.905 & \textbf{0.956} & 0.907 \\  
vote & 0.948 & \textbf{0.957} & \textbf{0.957} & \textbf{0.957} & 0.948 & \textbf{0.957} & 0.950 & 0.956 & 0.944 & 0.944 & 0.943 & \textbf{0.948} \\  
yeast & 0.687 & 0.687 & \textbf{0.702} & \textbf{0.702} & \textbf{0.705} & 0.687 & 0.690 & 0.696 & 0.694 & 0.694 & 0.692 & \textbf{0.702} \\  
zoo-1 & 0.995 & 0.995 & 0.995 & \textbf{1.000} & \textbf{0.995} & \textbf{0.995} & \textbf{0.995} & \textbf{0.995} & \textbf{0.995} & \textbf{0.995} & \textbf{0.995} & \textbf{0.995} \\ 
\hline
Average & \textbf{0.846} & 0.822 & 0.834 & 0.835 & \textbf{0.868} & 0.854 & 0.866 & 0.867 & 0.880 & 0.871 & \textbf{0.883} & 0.879 \\
\hline
\end{tabular}
\end{table}

The {\ttfamily bsnsing} package also compares favorably to DL8.5 in training speed, especially for large instances. Let us look at the letter data set which consists of 14000 training samples for example: it took {\ttfamily bsnsing} on average 39.5 seconds to train a model that worked better than the DL8.5 model, which took 304.4 seconds to train on average. 
It has been shown in prior work, i.e., \citep{Aglin_Nijssen_Schaus_2020}, that DL8.5 runs orders of magnitude faster than several other optimal decision tree methods, including the original DL8 algorithm \citep{10.1145/1281192.1281250}, a constrained programming (CP)-based method \citep{CPtree2020}, and an MIP-based method BinOCT \citep{VerwerZhang2019}. Hence, we can infer that {\ttfamily bsnsing} must also outrun those other methods by a substantial margin. Overall, it is reasonable to claim that {\ttfamily bsnsing}, bearing comparable prediction accuracy, stands out in training speed among decision tree methods that involve solving mathematical optimization problems in the training process. 

In most decision tree induction methods, depth (i.e., level distance between the root node and the deepest node in a tree) is a hyperparameter for adjusting the classifier's complexity with regard to the bias-variance tradeoff. Heuristic methods such as CART typically realize depth control via tree pruning, while most ODT methods can explicitly constrain the maximum depth in the optimization models. 
However, the {\ttfamily bsnsing} method 
does not endogeneously handle a constraint on the maximum depth. The ``Dp'' column in Table \ref{tb:dl85_bsnsing_compare} lists the mode (most frequent value out of the 20 runs) of the depth of the {\ttfamily bsnsing} trees. 
We can see that, compared to trees built by DL8.5 which have a default ``max\_depth'' of 3, the trees built by {\ttfamily bsnsing} generally reach deeper, though not all branches extend to the same depth. 

Some users might take depth as a proxy for interpretability of a decision tree - shallower trees, or trees with fewer leaves, are deemed more interpretable than deeper trees. 
To facilitate performance comparison with depth-constrained ODT trees, we can naively prune a {\ttfamily bsnsing} tree so as to keep the number of leaf nodes below that of a binary tree of a given depth. For instance, a tree of depth 1, 2 and 3 would have at most 2, 4 and 8 leaf nodes, respectively. Using this method, we repeat the above experiments (over the 24 binary classification data sets, 20 runs for each) under different maximum depth values. The average out-of-sample accuracies are reported in Table \ref{tb:dl85_compare_depth1}, with the best value in each depth group highlighted in boldface. Two {\ttfamily max.rules} settings are tested for {\ttfamily bsnsing}: the default setting with {\ttfamily max.rules}=2, reported in column bs(2), and the {\ttfamily max.rules}=1 setting, reported in column bs(1). Since the GOSDT package does not have a parameter to limit the maximum depth or the number leaves, it is not part of the experimentation. Also, all the OSDT runs with depth = 2 and 3 have hit the 5-min time limit\footnote{Without the time limit, OSDT would in many cases exhaust the computer memory before terminating.}.  

We can see that the pruned {\ttfamily bsnsing} trees remain quite competitive, in many cases outperforming the DL8.5 and OSDT trees of the same depth. The multivariate splits (with {\ttfamily max.rules}=2) clearly give {\ttfamily bsnsing} an advantage in these comparisons. Under {\ttfamily max.rules}=1, the pruned {\ttfamily bsnsing} trees become least accurate in most cases. Though the {\ttfamily max.rules}=1 setting along with the naive pruning is not recommended for {\ttfamily bsnsing}'s practical use, comparing bs(1) with DL8.5 and OSDT does highlight the benefits of holistic optimization in tree induction, as argued in several ODT papers.  Another interesting, yet expected, observation is that the constraint on depth, no matter how it is realized in different packages, does not affect the new-data prediction accuracy in any deterministic direction (increase or descrease). An optimal (unconstrained or unpruned) tree may perform worse than a depth-constrained (or naively pruned, as in the {\ttfamily bsnsing} case) tree in some cases. This observation enhances the understanding that in machine learning algorithms, the notion of optimality only applies to the training problem, not to the inference problem. In other words, there is no single algorithm or parameter setting that is best-performing in all cases. Moreover, comparing the average (across all data sets) accuracies of the pruned {\ttfamily bsnsing} trees in column bs(2) and the average accuracy (0.885) of the original {\ttfamily bsnsing} trees in Table \ref{tb:dl85_bsnsing_compare}, we can see that the naive pruning strategy generally hurts performance, upholding the effectiveness of {\ttfamily bsnsing}'s algorithm design and the default parameter setting.

\subsection{Comparison with other decision tree packages in R}
In this section, we compare the out-of-the-box performance (i.e.,  using all default options and no hyperparameter tuning) of the {\ttfamily bsnsing} package against several other decision tree packages, namely,  {\ttfamily C50},  {\ttfamily party},  {\ttfamily tree} and {\ttfamily rpart},  that are available on the Comprehensive R Archive Network (CRAN). 
Then, for those cases on which {\ttfamily bsnsing} performs poorly, we will demonstrate some simple methods to improve the performance. 

Data used in the benchmarking experiments include 57 data sets for binary classification and 18 data sets for multi-class classification.  Among these 75 data sets,  one (iris) is from the \emph{datasets} package, one (bank) is from a FICO-sponsored explainable machine learning challenge \citep{FICOdata}, two (compas and heloc) are from the ProPublica and Trusted-AI GitHub repositories, two (GlaucomaMVF and dystrophy) are from the \emph{ipred} package,  six (BreastCancer, Glass, smiley, spirals, xor and Sonar) are from / generated by the \emph{mlbench} package,  six (obli, grid, diam, circ, ring and sha88) are synthetic data sets for 2D pattern recognition (see Figure \ref{fig:shapes_plot}), and the remaining 58 data sets are sourced from the UCI Machine Learning Repository \citep{Dua:2019}.  The names,  the number of observations ($n$), the number of independent variables ($p$), (for binary-class) the rate of the minority class (MR), and (for multi-class) the number of target classes ($J$) are listed in Tables \ref{tb:binary_datasets} and  \ref{tb:mult_datasets} for binary and multi-class classification data sets, respectively. 
This collection covers most of the commonly used data sets for methodology benchmarking in the classification tree literature, hence the data collection itself can be useful for future research.
These data sets are accessible by name in the R environment once the {\ttfamily bsnsing} library is loaded.  The R scripts for conducting the subsequent experiments are not part of the library but will be published in a different repository.  

\begin{table}
\centering
\small
\caption{Binary classification data sets.} \label{tb:binary_datasets}
\begin{tabular}{l r r r@{\hskip 0.3in} | l r r r@{\hskip 0.3in} | l r r r}
\hline
Name & $n$ & $p$ & MR & Name & $n$ & $p$ & MR & Name & $n$ & $p$ & MR \\
\hline
acute1 & 120 & 6 & 0.492 & haberman & 306 & 3 & 0.265 & pima & 768 & 8 & 0.349 \\  
 acute2 & 120 & 6 & 0.417 & heart & 303 & 13 & 0.459 & Qsar & 1055 & 41 & 0.337 \\  
 Adult & 32561 & 13 & 0.241 & heloc & 10459 & 23 & 0.478 & relax & 182 & 12 & 0.286 \\  
 auto & 392 & 7 & 0.375 & hepatitis & 155 & 19 & 0.206 & retention & 10000 & 8 & 0.338 \\  
 bank & 45211 & 16 & 0.117 & HTRU2 & 17898 & 8 & 0.092 & ring & 600 & 2 & 0.500 \\  
 banknote & 1372 & 4 & 0.445 & ILPD & 583 & 10 & 0.286 & seismic & 1690 & 18 & 0.031 \\  
 birthwt & 189 & 9 & 0.312 & Ionos & 351 & 34 & 0.359 & sh88 & 600 & 2 & 0.500 \\  
 BreastCancer & 699 & 9 & 0.345 & magic04 & 19020 & 10 & 0.352 & Sonar & 208 & 60 & 0.466 \\  
 circ & 600 & 2 & 0.500 & mammo & 830 & 5 & 0.486 & spambase & 4601 & 57 & 0.394 \\  
 climate & 540 & 18 & 0.085 & Monks1 & 556 & 6 & 0.500 & SPECT & 267 & 22 & 0.206 \\  
 compas & 7214 & 52 & 0.451 & Monks2 & 601 & 6 & 0.343 & spirals & 600 & 2 & 0.500 \\  
 connect & 208 & 60 & 0.466 & Monks3 & 554 & 6 & 0.480 & statlog.a & 690 & 14 & 0.445 \\  
 credit & 690 & 15 & 0.445 & Mushroom & 8124 & 21 & 0.482 & thoraric & 470 & 16 & 0.149 \\  
 diam & 600 & 2 & 0.500 & norm3p10 & 600 & 10 & 0.500 & tictactoe & 958 & 9 & 0.347 \\  
 dystrophy & 209 & 9 & 0.359 & norm3p5 & 600 & 5 & 0.500 & titanic & 2201 & 3 & 0.323 \\  
 Echocard & 61 & 11 & 0.279 & obli & 600 & 2 & 0.500 & trans & 748 & 4 & 0.238 \\  
 Fertility & 100 & 9 & 0.120 & ozone1 & 2536 & 72 & 0.029 & votes & 435 & 16 & 0.386 \\  
 GlaucomaMVF & 170 & 66 & 0.500 & ozone8 & 2534 & 72 & 0.063 & wdbc & 569 & 30 & 0.373 \\  
 grid & 600 & 2 & 0.475 & parkins & 195 & 22 & 0.246 & wpbc & 198 & 33 & 0.237 \\  
\hline
\end{tabular}
\end{table}

\begin{figure}[!htb]
    \centering
        \includegraphics[width=0.95\linewidth]{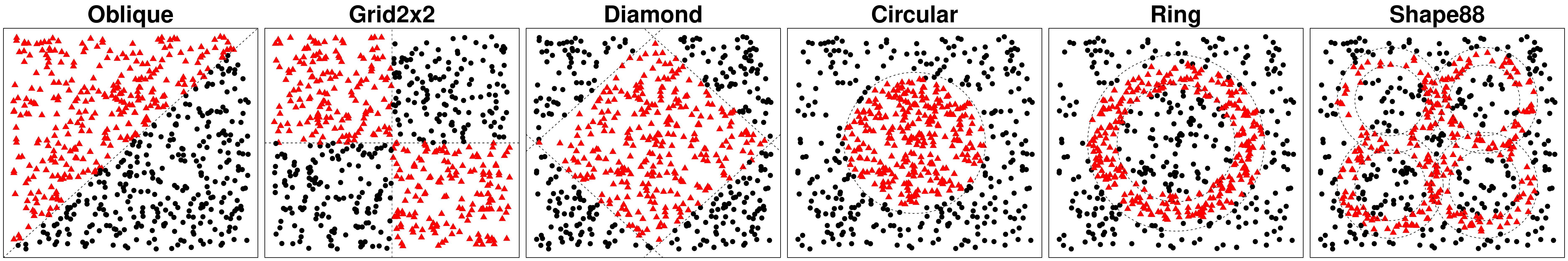}
\caption{Synthetic data sets for pattern recognition.  Input variables are the $x$ and $y$ coordinate.  Some slanted and nonlinear class boundaries are unamenable to the rectilinear split boundaries produced by tree models.} \label{fig:shapes_plot}
\end{figure}

\begin{table}
\centering
\small
\caption{Multi-class classification data sets.} \label{tb:mult_datasets}
\begin{tabular}{l r r r@{\hskip 0.3in} | l r r r@{\hskip 0.3in} | l r r r}
\hline
Name & $n$ & $p$ & $J$ & Name & $n$ & $p$ & $J$ & Name & $n$ & $p$ & $J$ \\
\hline
derm & 366 & 34 & 6 & imgsegm & 210 & 19 & 7 & thyroid & 3772 & 21 & 3 \\  
 iris & 150 & 4 & 3 & Hayes & 132 & 4 & 3 & wine & 178 & 13 & 3 \\  
 smiley & 500 & 2 & 4 & contra & 1473 & 9 & 3 & WineQuality & 4898 & 11 & 7 \\  
 xor3 & 600 & 3 & 4 & balance & 625 & 4 & 3 & nursery & 12960 & 8 & 5 \\  
 Glass & 214 & 9 & 6 & soybean.l & 266 & 35 & 15 &  & & & \\
 optdigits & 5620 & 64 & 10 & soybean.s & 47 & 35 & 4 &  & & & \\
 Seeds & 210 & 7 & 3 & tae & 151 & 5 & 3 &  & & & \\
\hline
\end{tabular}
\end{table}

\subsubsection{Out-of-the-box performance comparison.}

The experiments are conducted as follows.  For each data set, we randomly split all observations into two parts,  70\% for training and 30\% for testing.  The training set is fed into different decision tree functions to build the respective tree models,  then the models are fed into the {\ttfamily predict} functions of the respective packages to make predictions on the test set.  Accuracy and the area under the ROC curve (AUC) values are calculated for each method based on the prediction results.  To calculate the Accuracy,  class label predictions are requested from the predict functions,  and to calculate the AUC,  score (or probability) predictions are requested from the predict functions.   The whole process (i.e.,  random 70/30 split,  training and testing) is repeated 20 times with documented random number generator (RNG) seeds for each data set,  and the corresponding Mean Accuracy and Mean AUC (for binary classification) are reported in Tables \ref{tb:result_biclass} and \ref{tb:result_multclass}.  The average computing time in seconds of the {\ttfamily bsnsing} method is also reported under the CPU column in the tables.  The computing times of other methods are consistently below 1 second for all test cases, thus they are omitted from the report. 

\begin{table}
\centering
\small
\caption{Comparison on binary classification cases.} \label{tb:result_biclass}
\begin{tabular}{l r r r r r@{\hskip 0.3in}  r r r r r r}
\hline
 & \multicolumn{5}{c}{Mean Accuracy} & \multicolumn{5}{c}{Mean AUC} & \multirow{2}{*}{CPU} \\
\cline{2-11}
 & C5.0 & ctree & rpart & tree & bsnsing & C5.0 & ctree & rpart & tree & bsnsing & \\
\hline
acute1 &{\color{red}1.000} & 0.903 & 0.912 & 0.994 & 0.896 & {\color{red}1.000} & 0.957 & 0.932 & 0.996 & 0.970 & 0.1\\ 
acute2 &{\color{red}1.000} & 0.954 & 0.958 & {\color{red}1.000} & {\color{red}\textbf{1.000}} & {\color{red}1.000} & 0.989 & 0.965 & {\color{red}1.000} & {\color{red}\textbf{1.000}} & 0.0\\ 
Adult &{\color{red}0.864} & 0.847 & 0.832 & 0.832 & 0.825 & 0.887 & {\color{red}0.894} & 0.818 & 0.852 & \textbf{0.870} & 82.0 \\ 
auto &{\color{red}0.897} & 0.871 & 0.882 & 0.879 & 0.864 & {\color{red}0.953} & 0.901 & 0.938 & 0.923 & \textbf{0.941} & 0.3\\ 
bank &0.902 & {\color{red}0.904} & 0.901 & 0.890 & \textbf{0.901} & 0.880 & {\color{red}0.915} & 0.761 & 0.881 & \textbf{0.900} & 71.0 \\ 
banknote &{\color{red}0.981} & 0.967 & 0.966 & 0.978 & 0.963 & {\color{red}0.985} & 0.977 & 0.976 & 0.984 & {\color{red}\textbf{0.985}} & 0.5\\ 
birthwt &0.649 & {\color{red}0.674} & 0.661 & 0.626 & 0.629 & 0.541 & 0.486 & 0.581 & 0.584 & {\color{red}\textbf{0.587}} & 0.3\\ 
BreastCancer &0.942 & 0.946 & 0.940 & {\color{red}0.950} & \textbf{0.945} & 0.968 & 0.974 & 0.950 & 0.969 & {\color{red}\textbf{0.978}} & 0.3\\ 
circ &{\color{red}0.951} & 0.580 & 0.936 & 0.940 & \textbf{0.906} & 0.961 & 0.592 & 0.942 & {\color{red}0.966} & \textbf{0.956} & 0.2\\ 
climate &0.922 & 0.920 & {\color{red}0.929} & 0.921 & \textbf{0.925} & 0.810 & 0.809 & 0.809 & 0.757 & {\color{red}\textbf{0.813}} & 0.6\\ 
compas &0.890 & 0.889 & {\color{red}0.891} & {\color{red}0.891} & 0.890 & 0.918 & 0.934 & 0.884 & 0.919 & {\color{red}\textbf{0.934}} & 4.0\\ 
connect &0.702 & 0.694 & {\color{red}0.727} & 0.719 & \textbf{0.721} & 0.755 & 0.749 & 0.778 & 0.770 & {\color{red}\textbf{0.788}} & 2.0\\ 
credit &0.852 & {\color{red}0.852} & 0.851 & 0.840 & \textbf{0.851} & 0.897 & 0.909 & 0.902 & 0.890 & {\color{red}\textbf{0.913}} & 1.0\\ 
diam &{\color{red}0.923} & 0.477 & 0.890 & 0.906 & \textbf{0.906} & 0.947 & 0.503 & 0.919 & 0.944 & {\color{red}\textbf{0.953}} & 0.3\\ 
dystrophy &{\color{red}0.841} & 0.821 & 0.822 & 0.837 & 0.813 & 0.846 & 0.853 & 0.823 & {\color{red}0.857} & 0.838 & 0.2\\ 
Echocard &0.958 & 0.936 & {\color{red}0.966} & {\color{red}0.966} & 0.947 & 0.954 & 0.953 & 0.963 & {\color{red}0.967} & 0.934 & 0.0\\ 
Fertility &0.882 & {\color{red}0.893} & 0.865 & 0.850 & 0.860 & 0.545 & 0.535 & 0.578 & 0.635 & {\color{red}\textbf{0.664}} & 0.1\\ 
GlaucomaMVF &0.890 & 0.836 & 0.895 & 0.879 & {\color{red}\textbf{0.900}} & 0.933 & 0.894 & 0.946 & {\color{red}0.948} & {\color{red}\textbf{0.948}} & 1.0\\ 
grid &0.520 & 0.520 & {\color{red}0.976} & 0.610 & \textbf{0.972} & 0.505 & 0.505 & {\color{red}0.990} & 0.615 & {\color{red}\textbf{0.990}} & 0.1\\ 
haberman &{\color{red}0.736} & 0.717 & 0.726 & 0.714 & \textbf{0.729} & 0.544 & 0.619 & 0.640 & 0.649 & {\color{red}\textbf{0.672}} & 0.3\\ 
heart &0.777 & 0.752 & {\color{red}0.792} & 0.766 & 0.768 & 0.811 & 0.805 & {\color{red}0.822} & 0.808 & {\color{red}{\textbf{0.822}}} & 0.4\\ 
heloc &{\color{red}0.706} & 0.696 & 0.700 & 0.697 & 0.698 & 0.749 & 0.757 & 0.706 & 0.736 & {\color{red}\textbf{0.758}} & 37.0 \\ 
hepatitis &0.795 & 0.791 & 0.789 & {\color{red}0.808} & \textbf{0.798} & 0.716 & 0.706 & 0.677 & 0.716 & {\color{red}\textbf{0.764}} & 0.2\\ 
HTRU2 &{\color{red}0.979} & 0.979 & 0.978 & 0.977 & 0.977 & 0.948 & {\color{red}0.974} & 0.909 & 0.969 & \textbf{0.966} & 23.0 \\ 
ILPD &0.678 & {\color{red}0.704} & 0.681 & 0.673 & 0.682 & 0.675 & 0.665 & 0.657 & {\color{red}0.681} & \textbf{0.677} & 0.8\\ 
Ionos &0.892 & {\color{red}0.905} & 0.870 & 0.874 & 0.859 & {\color{red}0.920} & 0.901 & 0.901 & 0.901 & 0.894 & 1.0 \\ 
magic04 &{\color{red}0.850} & 0.844 & 0.819 & 0.814 & \textbf{0.840} & 0.885 & 0.892 & 0.811 & 0.842 & {\color{red}\textbf{0.893}} & 106.0 \\ 
mammo &0.828 & 0.807 & {\color{red}0.830} & 0.823 & \textbf{0.824} & 0.869 & 0.855 & 0.868 & 0.878 & {\color{red}\textbf{0.888}} & 0.7\\ 
Monks1 &{\color{red}0.898} & 0.743 & 0.840 & 0.743 & \textbf{0.871} & 0.899 & 0.739 & 0.916 & 0.739 & {\color{red}\textbf{0.956}} & 0.3\\ 
Monks2 &{\color{red}0.925} & 0.650 & 0.750 & 0.656 & 0.607 & {\color{red}0.973} & 0.492 & 0.800 & 0.540 & 0.598 & 0.7\\ 
Monks3 &{\color{red}0.989} & 0.961 & 0.977 & 0.986 & 0.960 & 0.987 & 0.983 & 0.979 & {\color{red}0.989} & \textbf{0.986} & 0.2\\ 
Mushroom &{\color{red}1.000} & 0.999 & 0.994 & 0.999 & 0.987 & 1.000 & {\color{red}1.000} & 0.994 & 0.999 & \textbf{0.999} & 4.0 \\ 
norm3p10 &{\color{red}0.768} & 0.749 & 0.758 & 0.766 & 0.756 & 0.823 & 0.814 & 0.801 & 0.815 & {\color{red}\textbf{0.831}} & 1.0 \\ 
norm3p5 &{\color{red}0.838} & 0.829 & 0.831 & 0.836 & \textbf{0.835} & 0.878 & 0.877 & 0.865 & {\color{red}0.896} & \textbf{0.892} & 0.5\\ 
obli &{\color{red}0.945} & 0.922 & 0.914 & 0.935 & 0.902 & {\color{red}0.962} & 0.955 & 0.942 & 0.957 & \textbf{0.961} & 0.2\\ 
ozone1 &0.969 & {\color{red}0.972} & 0.964 & 0.956 & 0.956 & 0.614 & {\color{red}0.784} & 0.674 & 0.632 & \textbf{0.773} & 6.0 \\ 
ozone8 &0.930 & {\color{red}0.932} & 0.930 & 0.923 & 0.922 & 0.777 & {\color{red}0.805} & 0.750 & 0.712 & \textbf{0.795} & 17.0 \\ 
parkins &0.854 & 0.847 & 0.861 & {\color{red}0.868} & \textbf{0.858} & 0.837 & 0.819 & 0.854 & {\color{red}0.868} & 0.842 & 0.4\\ 
pima &0.741 & {\color{red}0.746} & 0.740 & 0.741 & 0.731 & 0.772 & 0.780 & 0.779 & {\color{red}0.787} & \textbf{0.780} & 0.9\\ 
Qsar &{\color{red}0.839} & 0.807 & 0.822 & 0.818 & 0.819 & 0.862 & 0.850 & 0.841 & 0.860 & {\color{red}\textbf{0.872}} & 5.0 \\ 
relax &{\color{red}0.726} & {\color{red}0.726} & 0.594 & 0.612 & 0.617 & 0.492 & 0.492 & 0.483 & 0.503 & {\color{red}\textbf{0.509}} & 0.4\\ 
retention &{\color{red}0.994} & 0.971 & 0.941 & 0.929 & 0.935 & {\color{red}0.999} & 0.989 & 0.950 & 0.966 & \textbf{0.984} & 6.0 \\ 
ring &0.845 & 0.482 & 0.861 & {\color{red}0.881} & \textbf{0.782} & 0.866 & 0.498 & 0.896 & {\color{red}0.918} & \textbf{0.856} & 0.5\\ 
seismic &{\color{red}0.970} & {\color{red}0.970} & {\color{red}0.970} & 0.959 & {\color{red}\textbf{0.970}} & 0.499 & {\color{red}0.625} & 0.499 & 0.568 & \textbf{0.570} & 1.0 \\ 
sh88 &0.737 & 0.477 & 0.812 & {\color{red}0.813} & 0.674 & 0.753 & 0.493 & 0.847 & {\color{red}0.851} & \textbf{0.745} & 0.6\\ 
Sonar &0.702 & 0.694 & {\color{red}0.727} & 0.719 & \textbf{0.721} & 0.755 & 0.749 & 0.778 & 0.770 & {\color{red}\textbf{0.788}} & 2.0 \\ 
spambase &{\color{red}0.924} & 0.906 & 0.896 & 0.902 & \textbf{0.911} & 0.957 & 0.949 & 0.899 & 0.948 & {\color{red}\textbf{0.960}} & 27.0 \\ 
SPECT &0.814 & 0.792 & {\color{red}0.829} & 0.825 & \textbf{0.821} & 0.786 & 0.721 & 0.782 & 0.785 & {\color{red}\textbf{0.803}} & 0.3\\ 
spirals &{\color{red}0.948} & 0.655 & 0.924 & 0.942 & 0.783 & 0.960 & 0.650 & 0.948 & {\color{red}0.962} & 0.847 & 0.5\\ 
statlog.a &0.847 & {\color{red}0.857} & 0.855 & 0.848 & 0.851 & 0.903 & 0.907 & 0.904 & 0.906 & {\color{red}\textbf{0.912}} & 0.8\\ 
thoraric &0.842 & {\color{red}0.850} & 0.831 & 0.791 & 0.817 & 0.505 & 0.517 & 0.524 & {\color{red}0.560} & \textbf{0.539} & 0.7\\ 
tictactoe &{\color{red}0.923} & 0.825 & 0.901 & 0.880 & 0.803 & {\color{red}0.974} & 0.915 & 0.961 & 0.953 & 0.869 & 1.0 \\ 
titanic &0.777 & {\color{red}0.787} & 0.781 & 0.782 & \textbf{0.782} & 0.705 & {\color{red}0.748} & 0.710 & 0.710 & \textbf{0.725} & 0.2\\ 
trans &0.763 & 0.761 & {\color{red}0.775} & 0.770 & 0.765 & 0.680 & 0.690 & {\color{red}0.709} & 0.688 & 0.688 & 0.6\\ 
votes &{\color{red}0.960} & 0.957 & 0.949 & 0.954 & 0.943 & {\color{red}0.986} & 0.978 & 0.962 & 0.984 & \textbf{0.980} & 0.2\\ 
wdbc &0.939 & 0.933 & 0.927 & 0.933 & {\color{red}\textbf{0.944}} & 0.964 & 0.959 & 0.940 & 0.956 & {\color{red}\textbf{0.967}} & 1.0 \\ 
wpbc &0.722 & 0.726 & 0.692 & 0.692 & {\color{red}\textbf{0.730}} & 0.562 & 0.592 & 0.586 & 0.586 & {\color{red}\textbf{0.625}} & 0.9\\ 
\hline
\textbf{Average} & 0.859 & 0.811 & 0.853 & 0.844 & 0.841 & 0.827 & 0.787 & 0.825 & 0.825 & 0.841 & 7.3 \\
\hline
\end{tabular}
\end{table}

\begin{table}
\centering
\small
\caption{Comparison on multi-class classification cases.} \label{tb:result_multclass}
\begin{tabular}{l@{\hskip 0.3in} r r r@{\hskip 0.3in}  r r r r r r}
\hline
 &  \multirow{2}{*}{$n$} & \multirow{2}{*}{$p$} & \multirow{2}{*}{$J$} & \multicolumn{5}{c}{Mean Accuracy} & \multirow{2}{*}{CPU} \\
\cline{5-9}
 & & & & C5.0 & ctree & rpart & tree & bsnsing & \\
\hline
balance & 625 & 4 & 3 & 0.782 & 0.776 & 0.771 & 0.769 & {\color{red}\textbf{0.816}} & 1.4 \\ 
contra & 1473 & 9 & 3 & 0.514 & 0.537 & {\color{red}0.547} & 0.520 & 0.525 & 5.0 \\ 
derm & 366 & 34 & 6 & {\color{red}0.949} & 0.934 & 0.928 & 0.923 & 0.915 & 1.6 \\ 
Glass & 214 & 9 & 6 & {\color{red}0.664} & 0.606 & 0.661 & 0.649 & \textbf{0.649} & 1.2 \\ 
Hayes & 132 & 4 & 3 & {\color{red}0.825} & 0.494 & 0.625 & 0.745 & \textbf{0.698} & 0.5 \\ 
imgsegm & 210 & 19 & 7 & {\color{red}0.864} & 0.784 & 0.843 & 0.848 & 0.828 & 1.1 \\ 
iris & 150 & 4 & 3 & 0.934 & 0.941 & 0.930 & 0.936 & {\color{red}\textbf{0.945}} & 0.2 \\ 
nursery & 12960 & 8 & 5 & {\color{red}0.992} & 0.974 & 0.874 & 0.858 & 0.903 & 21.7 \\ 
optdigits & 5620 & 64 & 10 & {\color{red}0.902} & 0.844 & 0.768 & 0.774 & \textbf{0.897} & 70.8 \\ 
Seeds & 210 & 7 & 3 & 0.912 & 0.879 & 0.900 & {\color{red}0.922} & \textbf{0.922} & 0.3 \\ 
smiley & 500 & 2 & 4 & 0.990 & 0.986 & {\color{red}0.992} & 0.992 & 0.985 & 0.2 \\ 
soybean.l & 266 & 35 & 15 & {\color{red}0.886} & 0.699 & 0.655 & 0.779 & 0.631 & 2.3 \\ 
soybean.s & 47 & 35 & 4 & 0.975 & 0.482 & 0.579 & 0.950 & {\color{red}\textbf{0.986}} & 0.2 \\ 
tae & 151 & 5 & 3 & {\color{red}0.507} & 0.379 & 0.478 & 0.489 & 0.455 & 0.8 \\ 
thyroid & 3772 & 21 & 3 & {\color{red}0.997} & 0.993 & 0.996 & 0.997 & 0.962 & 1.6 \\ 
wine & 178 & 13 & 3 & 0.920 & 0.897 & 0.889 & {\color{red}0.921} & \textbf{0.911} & 0.3 \\ 
WineQuality & 4898 & 11 & 7 & {\color{red}0.574} & 0.532 & 0.532 & 0.513 & 0.534 & 39.3 \\ 
xor3 & 600 & 3 & 4 & 0.711 & 0.217 & {\color{red}0.956} & 0.709 & \textbf{0.931} & 1.1 \\ 
\hline
\textbf{Average} & & & & 0.828 & 0.720 & 0.773 & 0.794 & 0.810 & 8.3 \\
\hline
\end{tabular}
\end{table}

There are a few points to note: (1) in each run,  all five methods were fed with the same training and test sets, so the comparison was apple-to-apple; (2) all the original $p$ independent variables contained in each data set were used - no prior variable selection was done; (3) the tree-building functions from all the five packages were called in the simplest form,  i.e.,  only the ``regression formula'' and the training data set were supplied as arguments in the function calls,  to produce results that represent the out-of-the-box performance. 

In Tables \ref{tb:result_biclass} and \ref{tb:result_multclass},  the best Accuracy and AUC values in each data set were highlighted in red.  In addition,  the {\ttfamily bsnsing} results that were above average among the five methods were printed in boldface.  We can see that for binary classification tasks,  the {\ttfamily bsnsing} package is in the leading position under the AUC category - it won 28 cases out of 57, whereas tree, C50,  ctree and rpart won 14,  10,  8 and 3 cases, respectively.  The {\ttfamily bsnsing} package also scored above average in 48 data sets,  i.e., in 84\% of all cases.  Therefore,  as far as the AUC performance is concerned,  {\ttfamily bsnsing} should be the package of choice for binary classification tasks.   In terms of the Accuracy metric,  C50 was clearly the leading one,  winning 27 cases.  For practitioners,  we comment that AUC represents a model's ability to correctly rank order new data points according to their likelihood of belonging to the target class.  The specific score threshold for making classifications is usually application-dependent, e.g.,  depending on the comparative costs of making a FP claim versus making a FN claim about a given new case.  In contrast,  the classification accuracy measures the overall proportion of false claims,  i.e., by treating FP and FN claims with equal weight, at a chosen score threshold.  Therefore,  we remark that AUC is a more well-rounded performance metric than classification Accuracy for binary classifiers. 

While the OPT-G model is only applicable for binary classification,  the {\ttfamily bsnsing} function can handle multi-class classification tasks as well.  When more than two unique levels of the target variable are present in the training data set,  a binary classification tree is built for each level (as the positive class) versus all the other levels (as the negative class).  In the prediction stage,  a score (i.e., probability prediction) is produced from each tree and the target level having the greatest score will be assigned as the class label for the new case.  From Table \ref{tb:result_multclass} we can see that {\ttfamily bsnsing}'s multi-class performance is second only to C50. A caveat is that the current way {\ttfamily bsnsing} handles multi-class classification tasks is more of ensemble learning rather than the decision tree learning, and the model's interpretability is not preserved.

For {\ttfamily bsnsing}, the median time to train a binary classification model is 0.6 seconds, tallied over the 1140 training instances (i.e., 57 data sets,  20 instances each), and the median time to train a multi-class classification model is 1.2 seconds,  tallied over the 360 training instances.  The computing time is much more tolerable than most (if not all) MIP-based optimal classification tree methods.

\subsection{Usage notes of the {\ttfamily bsnsing} package	}

Alluding to the No Free Lunch Theorem \citep{585893}, no single machine learning algorithm is universally the best-performing algorithm for all problems. To be generally useful for classification problems, most decision tree algorithms allow users to control the behavior of the algorithm via choosing values for a number of hyperparameters. Such flexibility can be a double-edged sword to the usability of an algorithm. Seasoned users, most likely developers, can have the convenience of experimenting with the algorithm without changing the code, but ordinary users unconcerned of the internal workings of the underlying algorithm may find too many parameters perplexing. The large, sometimes infinite, value space of hyperparameters also presents practical challenges to automated parameter-tuning processes. For example, a clear-cut valley point of the generalization error curve in the bias-variance tradeoff analysis (see Chapter 2 of \cite{James:2014:ISL:2517747}) may be difficult to identify, especially when the available training samples are relatively few in a high-dimensional feature space.

To ease the usage, we provide some guidance for parameter selection for the {\ttfamily bsnsing} algorithm from an ordinary user's perspective. The most important parameters for {\ttfamily bsnsing} are {\ttfamily max.rules} and {\ttfamily node.size}. For {\ttfamily max.rules}, we recommend using the default value of 2 for a good balance between training speed and model performance. A higher value would increase the solution time of OPT-G particularly at the root node, as can be observed in Table \ref{tb:ENUM_neval}. In the meantime, a higher value would not necessarily translate to a better classification performance because of the heuristic nature of the recursive partitioning process. For {\ttfamily node.size}, we recommend using the default value of 0 first, meaning to set the minimum node size dynamically. A larger value of {\ttfamily node.size} would lead to a smaller (thus more interpretable) tree but might underfit the data, while a smaller value would lead to a bigger tree and might leave some true patterns undistinguished. If it is known that strong, learnable patterns exist in a data case, then manually setting {\ttfamily node.size} to a small value, i.e., some positive integer smaller than $\sqrt{n}$, is likely to improve the classification performance over the default setting. Lastly, if interpretability is unimportant, an ensemble of several {\ttfamily bsnsing} trees each trained with different hyperparameter values can effectively boost the performance. 

Let us look at some concrete examples. We notice from Table \ref{tb:result_biclass} that on three data sets, namely, Monks2, spirals and tictactoe, {\ttfamily bsnsing} performed especially poorly compared to the best-performing method. This suggests that discoverable patterns exist in these data sets and that {\ttfamily bsnsing} could be configured more flexible at discovering them. Indeed, if we reduce the {\ttfamily node.size} value, a significant improvement in the out-of-sample performance can be realized for three cases, as shown under the ``Improved'' columns in Table \ref{tb:bsnsing_tuning}. 
A greater improvement has also been achieved by the ensemble approach, in which we trained a total of nine trees with parameter combinations of {\ttfamily max.rules} $\in [1, 2, 3]$ and {\ttfamily node.size} $\in [0, 1, 10]$. The class membership prediction was the result of majority voting, and the score prediction was the average of the scores predicted by the nine trees in the ensemble. The total time (in seconds) taken to train the nine trees remains quite manageable, as shown in the column CPU in Table \ref{tb:bsnsing_tuning}. 
More detailed usage examples with sample code are provided in the (online) appendix.

\begin{table}
\centering
\small
\caption{Improve the {\ttfamily bsnsing} performance via changing {\ttfamily node.size} and using ensemble.} \label{tb:bsnsing_tuning}
\begin{tabular}{l r r@{\hskip 0.3in} r r r@{\hskip 0.3in} r r r}
\hline
& \multicolumn{2}{c}{Original} & \multicolumn{3}{c}{Improved} & \multicolumn{3}{c}{Ensemble} \\
\cline{2-9}
& Accu & AUC & Accu & AUC & Parameter & Accu & AUC & CPU \\
\hline
Monks2 & 0.607 & 0.598 & 0.895 & 0.898 & node.size=1 & 0.736 & 0.921 & 8.8 \\
spirals & 0.783 & 0.847 & 0.927 & 0.943 & node.size=3 & 0.911 & 0.974 & 5.0 \\
tictactoe & 0.803 & 0.869 & 0.911 & 0.934 & node.size=3 & 0.920 & 0.982 & 12.0 \\
\hline
\end{tabular}
\end{table}

\section{Conclusion and Future Work} \label{sec:conclusion}
In this paper,  we have proposed a new method for classification tree induction that combines mathematical optimization with the recursive partitioning framework.  The method optimally selects a boolean combination of multiple variables to maximize the Gini reduction in each node split.  The split optimization model is the first one that is able to maximize the well-justified but nonlinear Gini reduction metric in a mixed-integer linear program.  We have developed an efficient search algorithm to solve realistically regulated instances faster than commercial optimization solvers,  making the overall solution scheme,  as well as the R package {\ttfamily bsnsing},  more accessible to both the practitioners' and the developers' community.  Evaluation results have suggested that the {\ttfamily bsnsing} package can generate the best AUC performance among other decision tree packages in R,  at the cost of a median training time of a few seconds.

A central theme in the design of classification tree algorithms is making tradeoffs to strike a balance among competing objectives, such as speed,  accuracy and interpretability.  We believe that optimization modeling is no substitute for the recursive partitioning framework,  it, however, can alleviate some structural restrictions via answering key design questions in a new light.  
One of the benefits of using mathematical optimization in decision tree induction is that it makes the process more tractable and justifiable.  
As observed in previous works and in this paper,  properly regularized optimal trees do not lead to overfitting.  Therefore,  decision tree optimization is worthy of further development. 

Several aspects of the present work can be extended.  First,  the OPT-G model exhibits a clear sparsity pattern which may be exploited to expedite the solution.  For instance,  it is possible to adapt the bounding technique in the ENUM algorithm to the branch-and-bound framework via adding user cuts,  and to develop multiple branch-and-bound trees for parallel computing.  However, this would require the use of advanced callback functions which are not currently supported in R APIs (for both Gurobi and CPLEX).  Implementation in other languages could exploit these possibilities.  Second,  the feature binarization process is unoptimized,  and the actual utility of the candidate split rules are unquantified.  It is possible that a good proportion of the candidate rules are dominated by others hence need not be generated in the first place.  Future research could explore the column generation paradigm to generate high-value binary features on the fly during the optimal selection process.  

\appendix
\section*{Appendix: Usage demonstration of the {\ttfamily bsnsing} package}
To install the {\ttfamily bsnsing} package,  an R user can run this command:  {\ttfamily library(devtools); install\_github(`profyliu/bsnsing')}.  The following code snippet demonstrates a stylized use case of building and evaluating a decision tree model.  

\begin{lstlisting}
library(bsnsing)
set.seed(2021)  # Set seed for RNG in the sample() function
n <- nrow(BreastCancer)
trainset <- sample(1:n, 0.7*n)  # randomly sample 70% for training
testset <- setdiff(1:n, trainset)  # the remaining is for testing
# Build a tree to predict Class, using all default options
bs <- bsnsing(Class~., data = BreastCancer[trainset,])
summary(bs)  # display the tree structure, see Figure 6
pred <- predict(bs, BreastCancer[testset,], type='class')
actual <- BreastCancer[testset, 'Class']
table(pred, actual)  # display the confusion matrix
# Plot the ROC curve and display the AUC
ROC_func(data.frame(predict(bs, BreastCancer[testset,]), 
                    BreastCancer[testset,'Class']), 
         2, 1, pos.label = 'malignant', plot.ROC=T)
# Plot the tree to PDF file and generate the latex source code
plot(bs, file='../bsnsing_test/fig/BreastCancer.pdf')  # see Figure 7 left
\end{lstlisting}

The model summary (generated by line 8) prints out the tree structure as well as node information in plain text,  shown in Figure \ref{fig:BC_tree_summary}.  We can read from the printout,  for example,  that the root node (Node 0) is classified as 0 (benign) with probability 0.6585,  and that 100\% of all training observations fall in this node,  of which 167 observations are class 1 and 322 observations are class 0.  
The confusion matrix on the training set is given at the end of the summary print.   Detailed information of the {\ttfamily bsnsing} tree object can be accessed by the R command {\ttfamily str(bs)}. 

\begin{figure}[!htb]
\includegraphics[width=1\linewidth]{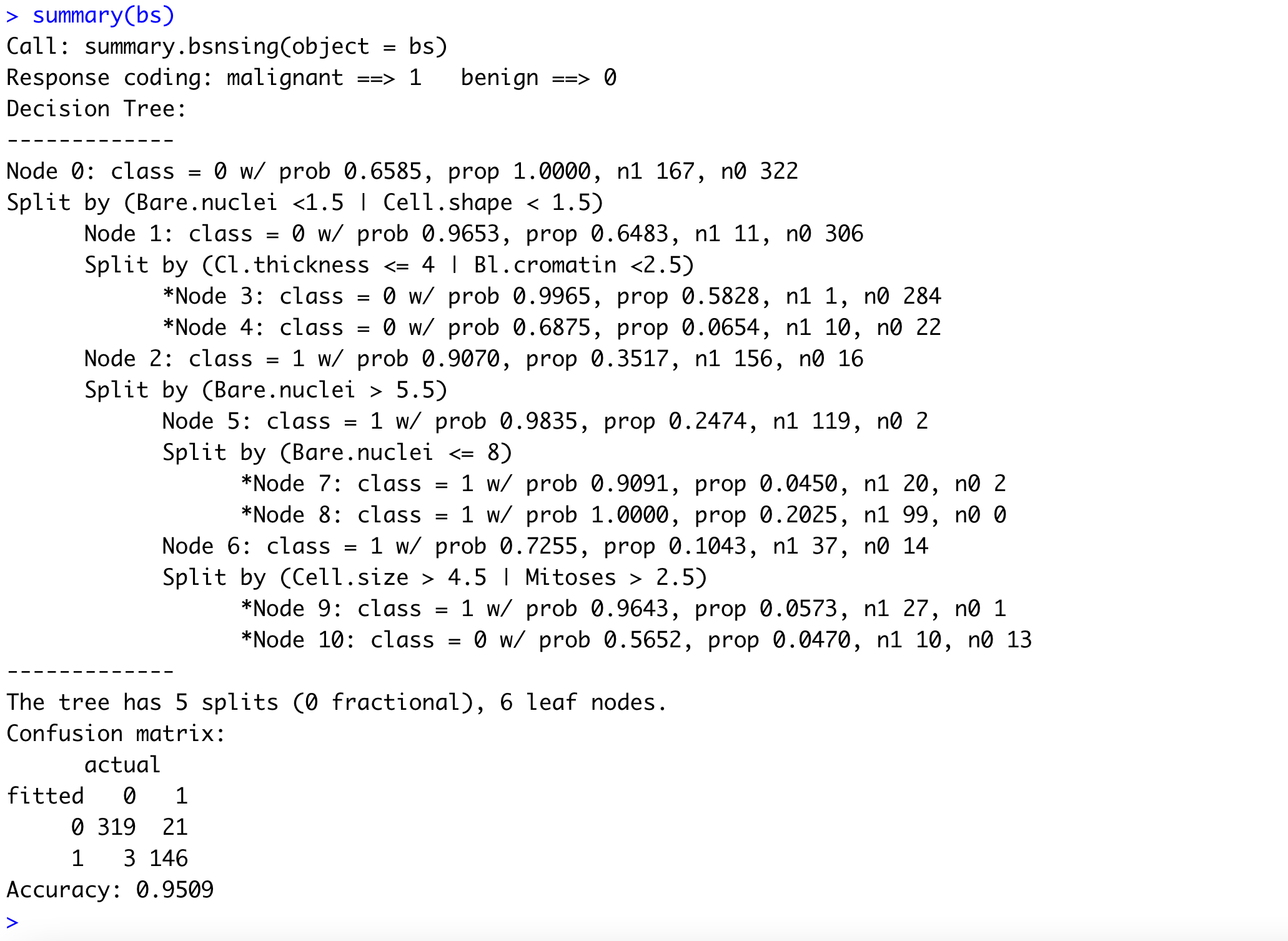}
\caption{Summary display of the {\ttfamily bsnsing} tree for the BreastCancer data set.} \label{fig:BC_tree_summary}
\end{figure}

\begin{figure}[!htb]
    \centering
    \begin{subfigure}[T]{0.7\textwidth}
        \includegraphics[width=1 \linewidth]{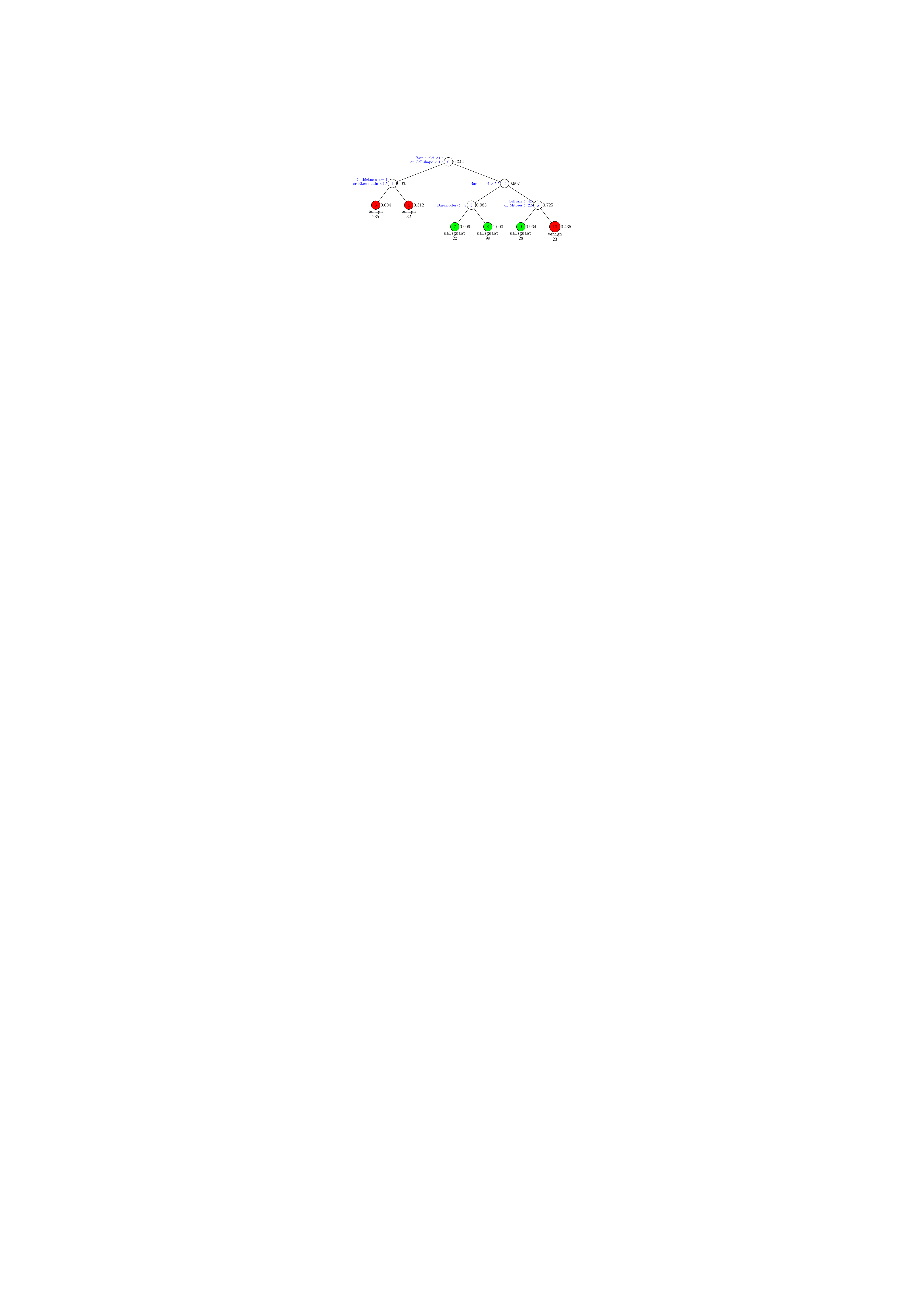}
    \end{subfigure}
    \quad
    \begin{subfigure}[T]{0.23\textwidth}
        \includegraphics[width=1 \linewidth]{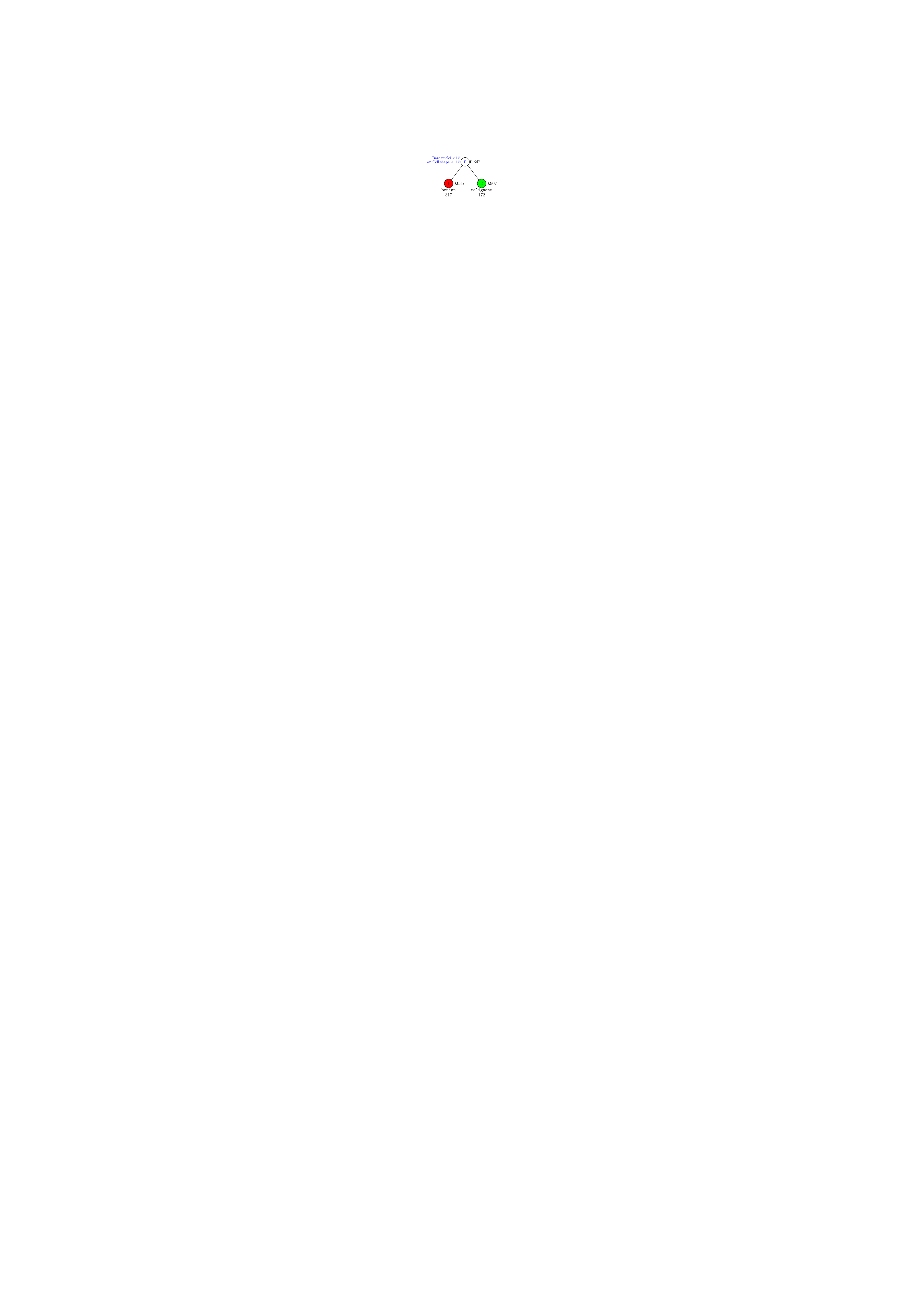}
    \end{subfigure}
\caption{Discriminability versus interpretability.  The left tree is built with the default options,  and the right tree is built with option {\ttfamily no.same.gender.children = True}.} \label{fig:BC_tree_demo}
\end{figure}

\begin{figure}[!htb]
\centering
\includegraphics[width=0.5\linewidth]{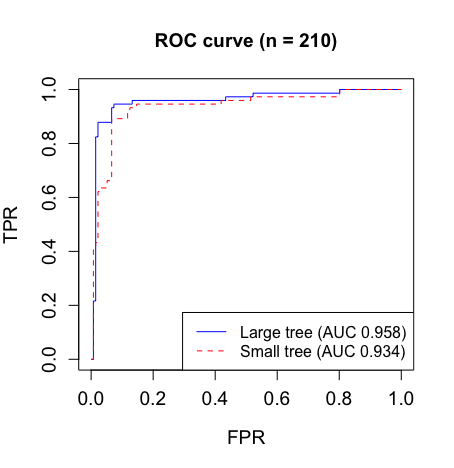}
\caption{ROC curves constructed on 210 test cases of the BreastCancer data set for two {\ttfamily bsnsing} trees.  The large tree was built with the default options,  and the small tree with option {\ttfamily no.same.gender.children = True}.} \label{fig:BC_tree_ROC}
\end{figure}

The {\ttfamily bsnsing} package implements the S3 method {\ttfamily plot} for plotting the {\ttfamily bsnsing} object (see line 17).  If a file name is provided (as shown in code), the function will save the latex scripts (that utilize the {\ttfamily tikz} package) to a .tex file and attempt to build the .ps and .pdf files by calling {\ttfamily latex,  dvips,  ps2pdf} commands if they are installed.  The plot is shown in Figure \ref{fig:BC_tree_demo} left.   Each node is represented by a circle with the node number printed inside the circle.  The color of a leaf node indicates its predicted class,  green for positive (class 1) and red for negative (class 0).  The split rule is shown on the left of each internal node in blue,  and the class 1 probability of the node is shown on the right of the node.  At the bottom of each leaf node,  the predicted label (in this case, malignant or benign),  along with the number of training observations that fall in the node,  is printed.  Of course,  these features can be easily customized and extended by other developers.  The right side of Figure \ref{fig:BC_tree_demo} plots a smaller tree generated on the same training set with the option {\ttfamily no.same.gender.children = True},  to suppress splits that would generate child nodes having the same majority class.  Figure \ref{fig:BC_tree_ROC} compares the ROC curves of these two trees.  In these particular cases,  some substantial improvement in interpretability was only accompanied by a slight drop in AUC,  so the smaller tree (in the author's opinion) is preferred.  Try-and-compare is a common practice in predictive analytics,  and the {\ttfamily bsnsing} library is generally fast enough and flexible enough to support such practice.


\bibliographystyle{plainnat}  
\bibliography{treebib} 

\begin{thebibliography}{43}
\providecommand{\natexlab}[1]{#1}
\providecommand{\url}[1]{\texttt{#1}}
\expandafter\ifx\csname urlstyle\endcsname\relax
  \providecommand{\doi}[1]{doi: #1}\else
  \providecommand{\doi}{doi: \begingroup \urlstyle{rm}\Url}\fi

\bibitem[Aghaei et~al.(2020)Aghaei, Gomez, and Vayanos]{aghaei2020learning}
Sina Aghaei, Andres Gomez, and Phebe Vayanos.
\newblock Learning optimal classification trees: Strong max-flow formulations,
  2020.

\bibitem[Aglin et~al.(2020)Aglin, Nijssen, and
  Schaus]{Aglin_Nijssen_Schaus_2020}
Gaël Aglin, Siegfried Nijssen, and Pierre Schaus.
\newblock Learning optimal decision trees using caching branch-and-bound
  search.
\newblock \emph{Proceedings of the AAAI Conference on Artificial Intelligence},
  34\penalty0 (04):\penalty0 3146--3153, Apr. 2020.
\newblock \doi{10.1609/aaai.v34i04.5711}.
\newblock URL \url{https://ojs.aaai.org/index.php/AAAI/article/view/5711}.

\bibitem[Agrawal et~al.(1993)Agrawal, Imielinski, and
  Swami]{Agrawal93miningassociation}
Rakesh Agrawal, Tomasz Imielinski, and Arun Swami.
\newblock Mining association rules between sets of items in large databases.
\newblock In \emph{Proceedings of The 1993 ACM SIGMOD International Conference
  on Management of Data, Washington DC, USA}, pages 207--216, 1993.

\bibitem[Alaradi and Hilal(2020)]{9325614}
Mohamed Alaradi and Sawsan Hilal.
\newblock Tree-based methods for loan approval.
\newblock In \emph{2020 International Conference on Data Analytics for Business
  and Industry: Way Towards a Sustainable Economy (ICDABI)}, pages 1--6, 2020.
\newblock \doi{10.1109/ICDABI51230.2020.9325614}.

\bibitem[Angelino et~al.(2017)Angelino, Larus-Stone, Alabi, Seltzer, and
  Rudin]{10.1145/3097983.3098047}
Elaine Angelino, Nicholas Larus-Stone, Daniel Alabi, Margo Seltzer, and Cynthia
  Rudin.
\newblock Learning certifiably optimal rule lists.
\newblock In \emph{Proceedings of the 23rd ACM SIGKDD International Conference
  on Knowledge Discovery and Data Mining}, KDD '17, page 35–44, New York, NY,
  USA, 2017. Association for Computing Machinery.
\newblock ISBN 9781450348874.
\newblock \doi{10.1145/3097983.3098047}.
\newblock URL \url{https://doi.org/10.1145/3097983.3098047}.

\bibitem[Bertsimas and Dunn(2017)]{Bertsimas2017}
Dimitris Bertsimas and Jack Dunn.
\newblock Optimal classification trees.
\newblock \emph{Machine Learning}, 106\penalty0 (7):\penalty0 1039--1082, Jul
  2017.
\newblock ISSN 1573-0565.
\newblock \doi{10.1007/s10994-017-5633-9}.
\newblock URL \url{https://doi.org/10.1007/s10994-017-5633-9}.

\bibitem[Bertsimas et~al.(2012)Bertsimas, Chang, and
  Rudin]{Bertsimas12aninteger}
Dimitris Bertsimas, Allison Chang, and Cynthia Rudin.
\newblock An integer optimization approach to associative classification.
\newblock In \emph{In Proceedings of Neural Information Processing Systems},
  pages 269--277, 2012.

\bibitem[Bertsimas et~al.(2019)Bertsimas, Dunn, Pawlowski, and
  Zhuo]{doi:10.1287/ijoo.2018.0001}
Dimitris Bertsimas, Jack Dunn, Colin Pawlowski, and Ying~Daisy Zhuo.
\newblock Robust classification.
\newblock \emph{INFORMS Journal on Optimization}, 1\penalty0 (1):\penalty0
  2--34, 2019.
\newblock \doi{10.1287/ijoo.2018.0001}.
\newblock URL \url{https://doi.org/10.1287/ijoo.2018.0001}.

\bibitem[Borgelt(2012)]{WIDM:WIDM1074}
Christian Borgelt.
\newblock Frequent item set mining.
\newblock \emph{Wiley Interdisciplinary Reviews: Data Mining and Knowledge
  Discovery}, 2\penalty0 (6):\penalty0 437--456, 2012.
\newblock ISSN 1942-4795.
\newblock \doi{10.1002/widm.1074}.
\newblock URL \url{http://dx.doi.org/10.1002/widm.1074}.

\bibitem[Breiman et~al.(1984)Breiman, Friedman, Stone, and
  Olshen]{breiman1984book}
Leo Breiman, Jerome Friedman, Charles~J. Stone, and R.A. Olshen.
\newblock \emph{Classification and Regression Trees}.
\newblock Taylor \& Francis, January 1984.

\bibitem[Dua and Graff(2017)]{Dua:2019}
Dheeru Dua and Casey Graff.
\newblock {UCI} machine learning repository, 2017.
\newblock URL \url{http://archive.ics.uci.edu/ml}.

\bibitem[{FICO}(2018)]{FICOdata}
{FICO}.
\newblock Explainable machine learning challenge, 2018.
\newblock https://community.fico.com/s/explainable-machine-learning-challenge.

\bibitem[Ghiasi et~al.(2020)Ghiasi, Zendehboudi, and
  Mohsenipour]{GHIASI2020105400}
Mohammad~M. Ghiasi, Sohrab Zendehboudi, and Ali~Asghar Mohsenipour.
\newblock Decision tree-based diagnosis of coronary artery disease: Cart model.
\newblock \emph{Computer Methods and Programs in Biomedicine}, 192:\penalty0
  105400, 2020.
\newblock ISSN 0169-2607.
\newblock \doi{https://doi.org/10.1016/j.cmpb.2020.105400}.
\newblock URL
  \url{https://www.sciencedirect.com/science/article/pii/S0169260719308971}.

\bibitem[Goh and Rudin(2014)]{10.1145/2623330.2623648}
Siong~Thye Goh and Cynthia Rudin.
\newblock Box drawings for learning with imbalanced data.
\newblock In \emph{Proceedings of the 20th ACM SIGKDD International Conference
  on Knowledge Discovery and Data Mining}, KDD '14, page 333–342, New York,
  NY, USA, 2014. Association for Computing Machinery.
\newblock ISBN 9781450329569.
\newblock \doi{10.1145/2623330.2623648}.
\newblock URL \url{https://doi.org/10.1145/2623330.2623648}.

\bibitem[Holte(1993)]{Holte1993}
Robert~C. Holte.
\newblock Very simple classification rules perform well on most commonly used
  datasets.
\newblock \emph{Machine Learning}, 11\penalty0 (1):\penalty0 63--90, Apr 1993.
\newblock ISSN 1573-0565.
\newblock URL \url{https://doi.org/10.1023/A:1022631118932}.

\bibitem[Hu et~al.(2019)Hu, Rudin, and
  Seltzer]{DBLP:journals/corr/abs-1904-12847}
Xiyang Hu, Cynthia Rudin, and Margo~I. Seltzer.
\newblock Optimal sparse decision trees.
\newblock \emph{CoRR}, abs/1904.12847, 2019.
\newblock URL \url{http://arxiv.org/abs/1904.12847}.

\bibitem[Hyafil and Rivest(1976)]{HYAFIL197615}
Laurent Hyafil and Ronald~L. Rivest.
\newblock Constructing optimal binary decision trees is np-complete.
\newblock \emph{Information Processing Letters}, 5\penalty0 (1):\penalty0 15 --
  17, 1976.
\newblock ISSN 0020-0190.
\newblock \doi{https://doi.org/10.1016/0020-0190(76)90095-8}.
\newblock URL
  \url{http://www.sciencedirect.com/science/article/pii/0020019076900958}.

\bibitem[James et~al.(2014)James, Witten, Hastie, and
  Tibshirani]{James:2014:ISL:2517747}
Gareth James, Daniela Witten, Trevor Hastie, and Robert Tibshirani.
\newblock \emph{An Introduction to Statistical Learning: With Applications in
  R}.
\newblock Springer Publishing Company, Incorporated, 2014.
\newblock ISBN 1461471370, 9781461471370.

\bibitem[Kass(1980)]{10.2307/2986296}
G.~V. Kass.
\newblock An exploratory technique for investigating large quantities of
  categorical data.
\newblock \emph{Journal of the Royal Statistical Society. Series C (Applied
  Statistics)}, 29\penalty0 (2):\penalty0 119--127, 1980.
\newblock ISSN 00359254, 14679876.
\newblock URL \url{http://www.jstor.org/stable/2986296}.

\bibitem[Letham et~al.(2015)Letham, Rudin, McCormick, and Madigan]{letham2015}
Benjamin Letham, Cynthia Rudin, Tyler~H. McCormick, and David Madigan.
\newblock Interpretable classifiers using rules and bayesian analysis: Building
  a better stroke prediction model.
\newblock \emph{Ann. Appl. Stat.}, 9\penalty0 (3):\penalty0 1350--1371, 09
  2015.
\newblock URL \url{https://doi.org/10.1214/15-AOAS848}.

\bibitem[Lin et~al.(2020{\natexlab{a}})Lin, Zhong, Hu, Rudin, and
  Seltzer]{lin2020generalized}
Jimmy Lin, Chudi Zhong, Diane Hu, Cynthia Rudin, and Margo Seltzer.
\newblock Generalized and scalable optimal sparse decision trees.
\newblock In \emph{International Conference on Machine Learning}, pages
  6150--6160, 2020{\natexlab{a}}.

\bibitem[Lin et~al.(2020{\natexlab{b}})Lin, Zhong, Hu, Rudin, and
  Seltzer]{lin2020generalized_long}
Jimmy Lin, Chudi Zhong, Diane Hu, Cynthia Rudin, and Margo Seltzer.
\newblock Generalized and scalable optimal sparse decision trees,
  2020{\natexlab{b}}.

\bibitem[Liu et~al.(1998)Liu, Hsu, and Ma]{Liu1998}
B.~Liu, W.~Hsu, and Y.~Ma.
\newblock Integrating classification and association rule mining.
\newblock \emph{Knowledge discovery and data mining}, page 80–86, 1998.

\bibitem[Loh(2009)]{loh2009}
Wei-Yin Loh.
\newblock Improving the precision of classification trees.
\newblock \emph{Ann. Appl. Stat.}, 3\penalty0 (4):\penalty0 1710--1737, 12
  2009.
\newblock URL \url{https://doi.org/10.1214/09-AOAS260}.

\bibitem[Malioutov and Varshney(2013)]{pmlr-v28-malioutov13}
Dmitry Malioutov and Kush Varshney.
\newblock Exact rule learning via boolean compressed sensing.
\newblock In Sanjoy Dasgupta and David McAllester, editors, \emph{Proceedings
  of the 30th International Conference on Machine Learning}, volume 28:3 of
  \emph{Proceedings of Machine Learning Research}, pages 765--773, Atlanta,
  Georgia, USA, 17--19 Jun 2013. PMLR.
\newblock URL \url{http://proceedings.mlr.press/v28/malioutov13.html}.

\bibitem[Mandala et~al.(2012)Mandala, Nawangpalupi, and
  Praktikto]{MANDALA2012406}
I.~Gusti Ngurah~Narindra Mandala, Catharina~Badra Nawangpalupi, and
  Fransiscus~Rian Praktikto.
\newblock Assessing credit risk: An application of data mining in a rural bank.
\newblock \emph{Procedia Economics and Finance}, 4:\penalty0 406--412, 2012.
\newblock ISSN 2212-5671.
\newblock \doi{https://doi.org/10.1016/S2212-5671(12)00355-3}.
\newblock URL
  \url{https://www.sciencedirect.com/science/article/pii/S2212567112003553}.
\newblock International Conference on Small and Medium Enterprises Development
  with a Theme ?Innovation and Sustainability in SME Development? (ICSMED
  2012).

\bibitem[Nijssen and Fromont(2007)]{10.1145/1281192.1281250}
Siegfried Nijssen and Elisa Fromont.
\newblock Mining optimal decision trees from itemset lattices.
\newblock In \emph{Proceedings of the 13th ACM SIGKDD International Conference
  on Knowledge Discovery and Data Mining}, KDD '07, page 530–539, New York,
  NY, USA, 2007. Association for Computing Machinery.
\newblock ISBN 9781595936097.
\newblock \doi{10.1145/1281192.1281250}.
\newblock URL \url{https://doi.org/10.1145/1281192.1281250}.

\bibitem[Nijssen and Fromont(2010)]{Nijssen:2010:OCD:1830978.1830979}
Siegfried Nijssen and Elisa Fromont.
\newblock Optimal constraint-based decision tree induction from itemset
  lattices.
\newblock \emph{Data Mining and Knowledge Discovery}, 21\penalty0 (1):\penalty0
  9--51, July 2010.
\newblock ISSN 1384-5810.
\newblock URL \url{http://dx.doi.org/10.1007/s10618-010-0174-x}.

\bibitem[Quinlan and Cameron-Jones(1995)]{Quinlan:1995:OLS:1643031.1643032}
J.~R. Quinlan and R.~M. Cameron-Jones.
\newblock Oversearching and layered search in empirical learning.
\newblock In \emph{Proceedings of the 14th International Joint Conference on
  Artificial Intelligence - Volume 2}, IJCAI'95, pages 1019--1024, San
  Francisco, CA, USA, 1995. Morgan Kaufmann Publishers Inc.
\newblock ISBN 1-55860-363-8.
\newblock URL \url{http://dl.acm.org/citation.cfm?id=1643031.1643032}.

\bibitem[Quinlan(1993)]{Quinlan:1993:CPM:152181}
J.~Ross Quinlan.
\newblock \emph{C4.5: Programs for Machine Learning}.
\newblock Morgan Kaufmann Publishers Inc., San Francisco, CA, USA, 1993.
\newblock ISBN 1-55860-238-0.

\bibitem[Rijnbeek and Kors(2010)]{Rijnbeek:2010:FSA:1825171.1825175}
Peter~R. Rijnbeek and Jan~A. Kors.
\newblock Finding a short and accurate decision rule in disjunctive normal form
  by exhaustive search.
\newblock \emph{Machine Learning}, 80\penalty0 (1):\penalty0 33--62, July 2010.
\newblock ISSN 0885-6125.
\newblock URL \url{https://doi.org/10.1007/s10994-010-5168-9}.

\bibitem[Sorensen et~al.(2000)Sorensen, Miller, and Ooi]{Sorensen42}
Eric~H. Sorensen, Keith~L. Miller, and Chee~K. Ooi.
\newblock The decision tree approach to stock selection.
\newblock \emph{The Journal of Portfolio Management}, 27\penalty0 (1):\penalty0
  42--52, 2000.
\newblock ISSN 0095-4918.
\newblock \doi{10.3905/jpm.2000.319781}.
\newblock URL \url{https://jpm.pm-research.com/content/27/1/42}.

\bibitem[Street(2005)]{doi:10.1287/ijoc.1030.0047}
W.~Nick Street.
\newblock Oblique multicategory decision trees using nonlinear programming.
\newblock \emph{INFORMS Journal on Computing}, 17\penalty0 (1):\penalty0
  25--31, 2005.
\newblock \doi{10.1287/ijoc.1030.0047}.
\newblock URL \url{https://doi.org/10.1287/ijoc.1030.0047}.

\bibitem[Tan et~al.(2005)Tan, Steinbach, and Kumar]{IntroDM2005}
Pang-Ning Tan, Michael Steinbach, and Vipin Kumar.
\newblock \emph{Introduction to Data Mining}.
\newblock Pearson, 1st edition, 2005.

\bibitem[Tanner et~al.(2008)Tanner, Schreiber, Low, Ong, Tolfvenstam, Lai, Ng,
  Leo, Thi~Puong, Vasudevan, Simmons, Hibberd, and
  Ooi]{10.1371/journal.pntd.0000196}
Lukas Tanner, Mark Schreiber, Jenny G.~H. Low, Adrian Ong, Thomas Tolfvenstam,
  Yee~Ling Lai, Lee~Ching Ng, Yee~Sin Leo, Le~Thi~Puong, Subhash~G. Vasudevan,
  Cameron~P. Simmons, Martin~L. Hibberd, and Eng~Eong Ooi.
\newblock Decision tree algorithms predict the diagnosis and outcome of dengue
  fever in the early phase of illness.
\newblock \emph{PLOS Neglected Tropical Diseases}, 2\penalty0 (3):\penalty0
  1--9, 03 2008.
\newblock \doi{10.1371/journal.pntd.0000196}.
\newblock URL \url{https://doi.org/10.1371/journal.pntd.0000196}.

\bibitem[Verhaeghe et~al.(2020)Verhaeghe, Nijssen, Pesant, Quimper, and
  Schaus]{CPtree2020}
H\'{e}l\`{e}ne Verhaeghe, Siegfried Nijssen, Gilles Pesant, Claude-Guy Quimper,
  and Pierre Schaus.
\newblock Learning optimal decision trees using constraint programming.
\newblock \emph{Constraints}, 25:\penalty0 1--25, 12 2020.
\newblock \doi{10.1007/s10601-020-09312-3}.

\bibitem[Verwer and Zhang(2019{\natexlab{a}})]{VerwerZhang2019}
Sicco Verwer and Y.~Zhang.
\newblock Learning optimal classification trees using a binary linear program
  formulation.
\newblock In \emph{Proceedings of the Thirty-Third AAAI Conference on
  Artificial Intelligence (AAAI-19)}, pages 1625--1632. AAAI Press,
  2019{\natexlab{a}}.
\newblock ISBN 978-1-57735-809-1.
\newblock \doi{10.1609/aaai.v33i01.33011624}.
\newblock URL \url{https://aaai.org/Conferences/AAAI-19/}.
\newblock 33rd AAAI Conference on Artificial Intelligence, AAAI-19 ; Conference
  date: 27-01-2019 Through 01-02-2019.

\bibitem[Verwer and Zhang(2017)]{10.1007/978-3-319-59776-8_8}
Sicco Verwer and Yingqian Zhang.
\newblock Learning decision trees with flexible constraints and objectives
  using integer optimization.
\newblock In Domenico Salvagnin and Michele Lombardi, editors,
  \emph{Integration of AI and OR Techniques in Constraint Programming}, pages
  94--103, Cham, 2017. Springer International Publishing.
\newblock ISBN 978-3-319-59776-8.

\bibitem[Verwer and Zhang(2019{\natexlab{b}})]{Zhang_2019}
Sicco Verwer and Yingqian Zhang.
\newblock Learning optimal classification trees using a binary linear program
  formulation.
\newblock \emph{Proceedings of the AAAI Conference on Artificial Intelligence},
  33\penalty0 (01):\penalty0 1625--1632, Jul. 2019{\natexlab{b}}.
\newblock \doi{10.1609/aaai.v33i01.33011624}.
\newblock URL \url{https://ojs.aaai.org/index.php/AAAI/article/view/3978}.

\bibitem[Wang et~al.(2017)Wang, Rudin, Doshi-Velez, Liu, Klampfl, and
  MacNeille]{JMLR:v18:16-003}
Tong Wang, Cynthia Rudin, Finale Doshi-Velez, Yimin Liu, Erica Klampfl, and
  Perry MacNeille.
\newblock A bayesian framework for learning rule sets for interpretable
  classification.
\newblock \emph{Journal of Machine Learning Research}, 18\penalty0
  (70):\penalty0 1--37, 2017.
\newblock URL \url{http://jmlr.org/papers/v18/16-003.html}.

\bibitem[Wolpert and Macready(1997)]{585893}
D.H. Wolpert and W.G. Macready.
\newblock No free lunch theorems for optimization.
\newblock \emph{IEEE Transactions on Evolutionary Computation}, 1\penalty0
  (1):\penalty0 67--82, 1997.
\newblock \doi{10.1109/4235.585893}.

\bibitem[Yang et~al.(2017)Yang, Rudin, and Seltzer]{PMLR2017_ScalableBRL}
Hongyu Yang, Cynthia Rudin, and Margo Seltzer.
\newblock Scalable {B}ayesian rule lists.
\newblock In Doina Precup and Yee~Whye Teh, editors, \emph{Proceedings of the
  34th International Conference on Machine Learning}, volume~70 of
  \emph{Proceedings of Machine Learning Research}, pages 3921--3930,
  International Convention Centre, Sydney, Australia, 06--11 Aug 2017. PMLR.
\newblock URL \url{http://proceedings.mlr.press/v70/yang17h.html}.

\bibitem[Zhu et~al.(2020)Zhu, Murali, Phan, Nguyen, and
  Kalagnanam]{DBLP:conf/nips/ZhuMPNK20}
Haoran Zhu, Pavankumar Murali, Dzung~T. Phan, Lam~M. Nguyen, and Jayant
  Kalagnanam.
\newblock A scalable mip-based method for learning optimal multivariate
  decision trees.
\newblock In Hugo Larochelle, Marc'Aurelio Ranzato, Raia Hadsell,
  Maria{-}Florina Balcan, and Hsuan{-}Tien Lin, editors, \emph{Advances in
  Neural Information Processing Systems 33: Annual Conference on Neural
  Information Processing Systems 2020, NeurIPS 2020, December 6-12, 2020,
  virtual}, 2020.

\end{thebibliography}

\end{document}